\documentclass[letterpaper]{article}

\PassOptionsToPackage{numbers, compress}{natbib}
\usepackage[preprint]{neurips_2019}

\usepackage[utf8]{inputenc} 
\usepackage[T1]{fontenc}    
\usepackage{hyperref}       
\usepackage{url}            
\usepackage{booktabs}       
\usepackage{amsfonts}       
\usepackage{nicefrac}       
\usepackage{microtype}      

\usepackage{amssymb}
\usepackage{amsmath}
\usepackage{amsthm}
\usepackage{bm}
\usepackage{multirow}
\usepackage{dblfloatfix}
\usepackage{bbm}
\usepackage{booktabs}
\usepackage[algo2e]{algorithm2e}
\usepackage{color}

\usepackage{tikz}
\usetikzlibrary{shapes,decorations}
\usetikzlibrary{graphs,graphs.standard}
\usetikzlibrary{positioning}
\usepackage{subcaption} 

\newtheorem{thm}{Theorem}
\newtheorem{theorem}[thm]{Theorem}
\newtheorem{lem}[thm]{Lemma}
\newtheorem{corollary}[thm]{Corollary}
\newtheorem{proposition}[thm]{Proposition}
\newtheorem{definition}[thm]{Definition}

\newcommand{\graph}{\mathcal{G}}  
\newcommand{\spine}{\mathcal{S}}  
\newcommand{\rst}{\mathcal{R}}  
\newcommand{\cutG} [1] [\bm{u}] {\Phi_{\mathcal{G}} ({#1})}
\newcommand{\cutS} [1] [\bm{u}] {\Phi_{\spine} ({#1})}

\newcommand{\sset}{K}

\newcommand{\At} [1][t] {\mathcal{A}_{#1}}

\newcommand{\predyt}{\hat{y}_{t}}
\newcommand{\yt}{y_{t}}
\newcommand{\mut} [1][t] {\bmu_{#1}}
\newcommand{\bomt} [1][t] {\bom_{#1}}

\newcommand{\RE} [2] [\bmu_{t}] {d(#1, #2)}

\newcommand{\wt} [1][t]{\bm{w}_{#1}}
\newcommand{\wtm} [1][t]{\dot{\bm{w}}_{#1}}

\newcommand{\resistdiam}{R_\graph}

\newcommand{\omdot}[1][t]{\dot{\omega}_{{#1},\Specna}}
\newcommand{\bomdot}[1][t]{\dot{\bm{\omega}}_{#1}}

\newcommand{\mupredyt}[1][t]{\hat{y}_{\bmu_{#1}}}

\newcommand{\contset}[1][t]{\Psi_{#1}}
\newcommand{\newset}[1][t]{\Omega_{#1}}
\newcommand{\forgetset}[1][t]{\Theta_{#1}}
\newcommand{\bmutAt}[1][t]{\bmu_{#1}(\mathcal{A}_{#1})}
\newcommand{\bmutYt}[1][t]{\bmu_{#1}(\mathcal{Y}_{#1})}
\newcommand{\entropy}[1][t]{H(\bmu_{#1})}
\newcommand{\J}[1][\spset]{J_{#1}}

\newcommand{\cut}{\Phi}
\newcommand{\rcut}{\Phi^r}

\newcommand{\cutSpine}{\cut_{\spine}}

\newcommand{\transpose}{\top}

\newcommand{\Sp}{\alpha}
\newcommand{\pred}{\hat{y}}
\newcommand{\predt}{\hat{y_t}}
\newcommand{\Ct}[1][t]{\mathcal{Y}_{#1}}

\newcommand{\spset}{\mathcal{E}}    
\newcommand{\zzz}{\square}  
\newcommand{\sign}{\operatorname{sign}}
\newcommand{\Spec}[3]{\varepsilon^{#1,#2}_{#3}}
\newcommand{\Specna}{\varepsilon}
\newcommand{\se}{\Specna}

\newcommand{\FB}{\mathcal{F}}
\newcommand{\BTB}{\mathcal{B}}

\newcommand{\bK}{\bm{K}}
\newcommand{\bL}{\bm{L}}
\newcommand{\be}{\bm{e}}

\newcommand{\bp}{\bm{p}}
\newcommand{\bu}{\bm{u}}

\newcommand{\bom}{\bm{\omega}}
\newcommand{\bmu}{\bm{\mu}}

\newcommand{\cC}{\mathcal{C}}

\newcommand{\cG}{\mathcal{G}}
\newcommand{\cI}{\mathcal{I}}
\newcommand{\G}{\graph}
\newcommand{\cO}{\mathcal{O}}

\newcommand{\simplex}{{\Delta_n}}   

\newcommand{\into}{\rightarrow}

\newcommand{\argmin}{\operatornamewithlimits{argmin}}
\newcommand{\Ive}[1]{[#1]}

\newcommand{\FS}{SCS-F}
\newcommand{\BS}{SCS-B}

\newcommand{\SCS}{SCS}

\newcommand{\largediam}{Herbster-LargeDiameter}
\newcommand{\wta}{CGVZ10}

\newcommand{\pseminorm}{HL09}
\newcommand{\treeopt}{CesaBianchi-FastOptimalTrees}
\newcommand{\shazoo}{VCGZ11}
\newcommand{\belkinniyogimanifold}{BN04}

\newcommand{\harmoniczhu}{ZhuGL03}

\newcommand{\spec}{Freund-Specialists}
\newcommand{\nicolobook}{nicolobook}

\title{Online Prediction of Switching Graph Labelings with Cluster Specialists}

\author{%
	Mark Herbster \\
	Department of Computer Science\\
	University College London\\
	London \\
	United Kingdom\\
	\texttt{m.herbster@cs.ucl.ac.uk} \\
	\And
	James Robinson \\
	Department of Computer Science\\
	University College London\\
	London \\
	United Kingdom\\
	\texttt{j.robinson@cs.ucl.ac.uk}
}

\begin{document}
	
	\maketitle
	\begin{abstract}
		We address the problem of predicting the labeling of a graph in an online setting when the labeling is changing over time.  
		We present an algorithm based on a {\em specialist}~\cite{Freund-Specialists} approach; we develop the machinery of cluster specialists which probabilistically exploits the cluster structure in the graph. 	
		Our algorithm has two variants, one of which surprisingly only requires $\cO(\log n)$ time on any trial $t$ on an $n$-vertex graph, an exponential speed up over existing methods.
		We prove switching mistake-bound guarantees for both variants of our algorithm.
		Furthermore these mistake bounds {\em smoothly} vary with the magnitude of the change between successive labelings. 
		We perform experiments on Chicago Divvy Bicycle Sharing data and show that our algorithms significantly outperform an existing algorithm (a kernelized Perceptron) as well as several natural benchmarks.
	\end{abstract}
	
	\section{Introduction}\label{sec:introduction}
	We study the problem of predicting graph labelings that evolve over time.  Consider the following game for predicting the labeling of a graph in the online setting.  {\tt Nature} presents a graph $\cG$; {\tt Nature} queries a vertex $i_1 \in V = \{1,2,\ldots,n\}$; the {\tt learner} predicts the label of the vertex $\pred_1\in\{-1,1\}$; {\tt Nature} presents a label $y_1$; {\tt Nature} queries a vertex ${i_2}$; the {\tt learner} predicts $\pred_2$; and so forth. The {\tt learner}'s goal is to minimize the total number of mistakes $M = |\{t: \predt \ne y_t\}|$. If {\tt Nature} is strictly adversarial, the {\tt learner} will incur a mistake on every trial, but if {\tt Nature} is regular or simple, there is hope that the {\tt learner} may incur only  a few mistakes. Thus, a central goal of mistake-bounded online learning is to design algorithms whose total mistakes can be bounded relative to the complexity of {\tt Nature}'s labeling.  This (non-switching) graph labeling problem has been studied extensively in the online learning literature~\citep{\largediam,\pseminorm,CGVZ10,VCGZ11,HPG15}.  In this paper we generalize the setting to allow the underlying labeling to change arbitrarily over time.  The {\tt learner}  has no knowledge of when a change in labeling will occur and therefore must be able to adapt quickly to these changes.

	Consider an example of services placed throughout a city, such as public bicycle sharing stations. As the population uses these services the state of each station--such as the number of available bikes--naturally evolves throughout the day, at times gradually and others abruptly, and we might want to predict the state of any given station at any given time. Since the location of a given station as well as the state of nearby stations will be relevant to this learning problem it is natural to use a graph-based approach.
	Another setting might be a graph of major road junctions (vertices) connected by roads (edges), in which one wants to predict whether or not a junction is congested at any given time. Traffic congestion is naturally non-stationary and also exhibits both gradual and abrupt changes to the structure of the labeling over time~\citep{K98}.

	The structure of this paper is as follows. In Section~\ref{sec:background} we discuss the background literature.
	In Section~\ref{sec:switching-specialists} we present the {\sc Switching Cluster Specialists} algorithm (\SCS), a modification of the method of specialists~\citep{Freund-Specialists} with the novel machinery of \textit{cluster specialists}, a set of specialists that in a rough sense correspond to clusters in the graph. We consider two distinct sets of specialists, $\BTB_n$ and $\FB_n$, where $\BTB_n \subset \FB_n$. 
	With the smaller set of specialists the bound is only larger by factor of $\log n$.  On the other hand, prediction is exponentially faster per trial, remarkably requiring only $\mathcal{O}(\log n)$ time to predict.  
	In Section~\ref{sec:Experiments} we provide experiments on Chicago Divvy Bicycle Sharing data.  In Section~\ref{sec:conclusion} we provide some concluding remarks.  All proofs are contained in the technical appendices.
	
	\subsection{Notation}
	We first present common notation.
	Let $\graph=(V,E)$ be an undirected, connected, $n$-vertex graph with vertex set $V = \{1,2,\ldots,n\}$ and edge set $E$.   
	Each vertex of this graph may be labeled with one of two states $\{-1,1\}$ and thus a  labeling of a graph may be denoted by a vector $\bu\in\{-1,1\}^n$ where  $u_i$ denotes the label of vertex~$i$. 
	The underlying assumption is that we are predicting vertex labels from a sequence $\bu_1,\ldots,\bu_T\in\{-1,1\}^n$ of graph labelings over $T$ trials. 
	The set $\sset :=\{t \in \{2,\ldots,T\} :  \bu_t \neq \bu_{t-1}\}\cup\{1\}$ contains the first trial of each of the $|K|$ ``segments'' of the prediction problem.  Each segment corresponds to a time period when the underlying labeling is unchanging. 
	The {\em cut-size} of a labeling $\bu$ on a graph $\graph$ is defined as $\cut_{\graph}(\bu) := |\{(i,j)\in E : u_{i}\neq u_{j}\}|$, i.e., the number of edges between vertices of disagreeing labels. 
	
		
	We let $r_{\graph}(i,j)$ denote the \textit{resistance distance} (effective resistance) between vertices $i$ and $j$ when the graph $\graph$ is seen as a circuit where each edge has unit resistance (e.g.,~\cite{Klein-ResistanceDistance}). 
	The \textit{resistance diameter} of a graph is $\resistdiam := \max\limits_{i,j\in V}r_{\graph}(i,j)$.   The {\em resistance weighted} cut-size of a labeling $\bu$ is $\rcut_{\graph}(\bu) :=\!\!\!\!\!\!\!\!\!\!\sum\limits_{(i,j)\in E : u_{i}\neq u_{j}}\!\!\!\!\!\!\!\!\!\! r_{\graph}(i,j)$.      
	Let $\simplex = \left\{\bmu \in\left[0,1\right]^{n}: \sum_{i=1}^{n}\mu_{i} = 1\right\}$ be the $n$-dimensional probability simplex.  For $\bmu\in\simplex$ we define $\entropy[]:= \sum_{i=1}^{n}\mu_{i}\log_2{\frac{1}{\mu_i}}$ to be the entropy of $\bmu$. For $\bmu,\bom\in\simplex$ we define $\RE[\bmu]{\bom} = \sum_{i=1}^{n}\mu_{i}\log_2{\frac{\mu_{i}}{\omega_{i}}}$ to be the relative entropy between $\bmu$ and $\bom$. For a vector $\bom$ and a set of indices $\cI$ let $\bom(\cI) := \sum_{i\in\cI} \omega_i$. 
	For any positive integer $N$ we define $[N] := \left\{1,2,\ldots,N\right\}$ and for any predicate $[\mbox{\sc pred}] :=1$ if $\mbox{\sc pred}$
	is true and equals 0 otherwise.  

	\section{Related Work}\label{sec:background}
	The problem of predicting the labeling of a graph in the batch setting was introduced as a foundational method for semi-supervised (transductive) learning.  In this work, the graph was built using both the unlabeled and labeled instances.  The seminal work by~\citep{Blum:2001} used a metric on the instance space and then built a kNN or $\epsilon$-ball graph.  The partial labeling was then extended to the complete graph by solving a mincut-maxflow problem where opposing binary labels represented sources and sinks.
	In practice this method suffered from very unbalanced cuts.  Significant practical and theoretical advances were made by replacing the mincut/maxflow model with methods based on minimising a quadratic form of the graph Laplacian.  Influential early results include but are not limited to~\citep{\harmoniczhu,\belkinniyogimanifold,ZhouBLWS03}.   A limitation of the graph Laplacian-based techniques is that these batch methods--depending on their implementation--typically require $\Theta(n^2)$ to~$\Theta(n^3)$ time to produce a single set of predictions.  In the online switching setting we will aim for our fastest algorithm to have $\cO(\log n)$ time complexity per trial.
	
	Predicting the labeling of a graph in the online setting was introduced by~\cite{Herbster-OnlineLearningGraphs}.  
	The authors proved bounds for a Perceptron-like algorithm with a kernel based on the graph Laplacian.  Since this work there has been a number of extensions and improvements in bounds including but not limited to~\citep{\largediam,CesaBianchi-FastOptimalTrees,HL09,Herbster-SwitchingGraphs,HPG15,RS17}.
	Common to all of these papers is that a dominant term in their mistake bounds is the (resistance-weighted) cut-size.
	
	From a simplified perspective, the methods for predicting the labeling of a graph (online) split into two approaches. 
	The first approach works directly with the original graph and is usually based on a graph Laplacian~\citep{Herbster-OnlineLearningGraphs,\pseminorm,HPG15}; it provides bounds that utilize the additional connectivity of non-tree graphs, which are particularly strong when the graph contains uniformly-labeled clusters of small (resistance) diameter.  The drawbacks of this approach are that the bounds are weaker on graphs with large diameter, and that computation times are slower. 
	
	The second approach is to approximate the original graph with an appropriately selected tree or ``line'' graph~\citep{\largediam,\wta,\treeopt,\shazoo}.  This enables faster computation times, and bounds that are better on graphs with large diameters. 
	These algorithms may be extended to non-tree graphs by first selecting a spanning tree uniformly at random~\citep{\wta} and then applying the algorithm to the sampled tree. This randomized approach induces {\em expected} mistake bounds that also exploit the cluster structure in the graph (see Section \ref{subsec:random-spanning-trees-and-linearization}). Our algorithm takes this approach.

	\subsection{Switching Prediction}  
	In this paper rather than predicting a single labeling of a graph we instead will predict a (switching) sequence of labelings.   {\em Switching} in the mistake- or regret-bound setting refers to the problem of predicting an online sequence when the ``best comparator'' is changing over time.   In the simplest of switching models the set of comparators is {\em structureless} and we simply pay per switch.  A prominent early result in this model is~\cite{Herbster-TrackingExpert} which introduced the {\em fixed-share} update which will play a prominent role in our main algorithm. Other prominent results in the structureless model include but are not limited to~\citep{derandom,Bousquet-MPP,wacky,koolenswitch,kaw12,fixmir}.
	A stronger model is to instead prove a bound that holds for any arbitrary contiguous sequence of trials.  Such a bound is called an {\em adaptive-regret} bound. This type of bound automatically implies a bound on the structureless switching model.  Adaptive-regret was introduced in~\citep{HS07}\footnote{However, see the analysis of {\sc WML} in~\citep{LW94} for a precursory result.} other prominent results in this model include~\citep{AKCV12,fixmir,DGS15}.
		
	The structureless model may be generalized by introducing a divergence measure on the set of comparators.   Thus, whereas in the structureless model we pay for the number of switches, in the structured model we instead pay in the sum of divergences between successive  comparators.  This model was introduced in~\citep{Herbster-TrackLinPred}; prominent results include~\citep{KSW04,fixmir}.
		
	In~\citep{Herbster-SwitchingGraphs} the authors also consider switching graph label prediction.  However, their results are not directly comparable to ours since they consider the combinatorially more challenging problem of repeated switching within a small set of labelings contained in a larger set. That set-up was a problem originally framed in the ``experts'' setting and posed as an open problem by~\cite{freundopen} and solved in~\citep{Bousquet-MPP}.
	If we apply the bound in~\citep{Herbster-SwitchingGraphs} to the case where there is {\em not} repeated switching within a smaller set, then their bound is uniformly and significantly weaker than the bounds in this paper and the algorithm is quite slow requiring $\theta(n^3)$ time per trial in a typical implementation.
	Also contained in~\citep{Herbster-SwitchingGraphs} is a baseline algorithm based on a kernel perceptron with a graph Laplacian kernel.  The bound of that algorithm has the significant drawback in that it scales with respect to the ``worst'' labeling in a sequence of labelings.  However,  it is simple to implement and we use it as a benchmark in our experiments.

	\subsection{Random Spanning Trees and Linearization}\label{subsec:random-spanning-trees-and-linearization}
	Since we operate in the transductive setting where the entire unlabeled graph is presented to the {\tt learner} beforehand, this affords the {\tt learner} the ability to perform any reconfiguration to the graph as a preprocessing step. The bounds of most existing algorithms for predicting a labeling on a graph are usually expressed in terms of the cut-size of the graph under that labeling.  A natural approach then is to use a spanning tree of the original graph which can only reduce the cut-size of the labeling. 
	
	The effective resistance between vertices $i$ and $j$, denoted $r_{\graph}(i,j)$, is equal to the probability that a spanning tree of $\cG$ 
	drawn uniformly at random (from the set of all spanning trees of $\cG$) includes $(i,j)\in E$ as one of its $n-1$ edges (e.g., \cite{LP17}). As first observed by~\cite{\treeopt}, by selecting a spanning tree uniformly at random from the set of all possible spanning trees, mistake bounds expressed in terms of the cut-size then become {\em expected} mistake bounds now in terms of the effective-resistance-weighted cut-size of the graph. That is, if $\rst$ is a random spanning tree of $\G$ then $\mathbb{E}[\cut_{\rst}(\bu)] = \rcut_{\cG}(\bu)$ and thus $\rcut_{\graph}(\bu) \le \cut_{\graph}(\bu)$. A random spanning tree can be sampled from a graph efficiently using a random walk or similar methods (see e.g.,~\cite{Wilson-RST}). 
	
	To illustrate the power of this randomization consider the simplified example of a graph with two cliques each of size $\nicefrac{n}{2}$, where one clique is labeled uniformly with `+1' and the other `-1' with an additional  arbitrary $\nicefrac{n}{2}$ ``cut'' edges between the cliques.  This dense graph exhibits two disjoint clusters and $\cutG = \nicefrac{n}{2}$.  On the other hand $\rcut_{\cG}(\bu) = \Theta(1)$, since between any two vertices in the opposing cliques there are $\nicefrac{n}{2}$ edge disjoint paths of length $\le 3$ and thus the effective resistance  between any pair of vertices is $\Theta(\nicefrac{1}{n})$.  Since bounds usually scale linearly with (resistance-weighted) cut-size, the cut-size bound would be vacuous but the resistance-weighted cut-size bound would be small.
	
	We will make use of this preprocessing step of sampling a uniform random spanning tree, as well as a \textit{linearization} of this tree to produce a (spine) line-graph, $\spine$. 
	The linearization of $\graph$ to $\spine$ as a preprocessing step was first proposed by~\cite{Herbster-LargeDiameter} and has since been applied in, e.g.,~\citep{\wta,PSST16}. In order to construct $\spine$, a random-spanning tree $\rst$ is picked uniformly at random. A vertex of $\rst$ is then chosen and the graph is fully traversed using a \textit{depth-first search} generating an ordered list $V_{\mathcal{L}} = \left\{i_{l_{1}},\ldots,i_{l_{2m+1}}\right\}$ of vertices in the order they were visited. Vertices in $V$ may appear multiple times in $V_{\mathcal{L}}$. 
	A subsequence $V_{\mathcal{L}^{'}}\subseteq V_{\mathcal{L}}$ is then chosen such that each vertex in $V$ appears only once. The line graph $\spine$ is then formed by connecting each vertex in $V_{\mathcal{L'}}$ to its immediate neighbors in $V_{\mathcal{L'}}$ with an edge. We denote the edge set of $\spine$ by $E_{\spine}$ and let $\cut_{t}:= \cut(\bu_t)$, where the cut $\cut$ is with respect to the linear embedding $\spine$. Surprisingly, as stated in the lemma below, the cut on this linearized graph is no more than twice the cut on the original graph.
	
	\begin{lem}[\cite{Herbster-LargeDiameter}]
		\label{line-graph-cut-size-lemma}
		Given a labeling $\bu\in\{-1,1\}^{n}$ on a graph $\graph$, for the mapping $\graph \rightarrow \rst \rightarrow \spine$, as above, we have
		$\cutSpine(\bu) \leq 2\cut_{\rst}(\bu) \leq 2\cut_{\graph}(\bu)$.
	\end{lem}
	
	By combining the above observations we may reduce the problem of learning on a graph to that of learning on a line graph.  In particular, if we have an algorithm with a mistake bound of the form $M\le \mathcal{O}(\cut_{\cG}(\bu))$ this implies we then may give an {\em expected} mistake bound of the form $M\le \mathcal{O}(\cut^r_{\cG}(\bu))$ by first sampling a random spanning tree and then linearizing it as above.   
	Thus, for simplicity in presentation, we will only state the deterministic mistake bounds in terms of cut-size, although the expected bounds in terms of resistance-weighted cut-sizes will hold simultaneously.

	\section{Switching Specialists}\label{sec:switching-specialists}
	In this section we present a new method based on the idea of {\em specialists}~\citep{\spec} from the {\em prediction with expert advice} literature~\citep{LW94,AS90,\nicolobook}.   Although the achieved bounds are slightly worse than other methods for predicting a {\em single} labeling of a graph, the derived advantage is that it is possible to 
	obtain ``competitive'' bounds with fast algorithms to predict a sequence of changing graph labelings. 
	
	Our inductive bias is to predict well when a labeling has a {\em small} (resistance-weighted) cut-size.  The complementary perspective implies that the labeling consists of a {\em few} uniformly labeled  clusters.  This suggests the idea of maintaining a collection of basis functions where each such function is specialized to predict a constant function on a given cluster of vertices.   To accomplish this technically we adapt the method of {\em specialists}~\citep{\spec,kaw12}.   A specialist is a prediction function $\se$ from an input space to an extended output space with {\em abstentions}.  So for us the input space is just $V=[n]$, the vertices of a graph; and the extended output space is $\{-1,1,\zzz\}$ where $\{-1,1\}$ corresponds to predicted  labels of the vertices, but `$\zzz$' indicates that the specialist abstains from predicting.   Thus a specialist {\em specializes} its prediction to part of the input space and in our application the specialists correspond to a collection of clusters which cover the graph, each cluster uniformly predicting $-1$ or~$1$.
	\setlength{\textfloatsep}{3.0pt} 
	\begin{algorithm2e}[t]
		\begin{small}
			\SetAlgoVlined
			\DontPrintSemicolon
			\SetKwInOut{Input}{input}
			\SetKwInOut{Init}{initialize}
			\SetKwInOut{Parameter}{parameter }
			\SetKw{Predict}{predict}
			\SetKw{Receive}{receive}
			\SetKw{Set}{set}
			\Input{ Specialists set $\spset$}
			\Parameter{ $\alpha\in [0,1]$}
			\Init{ $\bom_1 \gets \frac{1}{|\spset|}\bm{1}$, $\bomdot[0] \gets \frac{1}{|\spset|}\bm{1}$, $\bp \gets \bm{0}$, $m \gets 0$}
			\For{$t=1$ \KwTo $T$}{
				\Receive{$i_{t}\in V$\;}
				
				\Set{ $\At := \{\Specna \in \spset : \Specna(i_t) \ne \zzz \}$\;}
				
				\ForEach(\tcp*[f]{delayed share update}){$\Specna\in\At$}{
					\vspace{-0.15in}
					\begin{flalign}\label{ShareUpdateLogN}
					\omega_{t,\se} \gets \left(1-\alpha\right)^{m-p_{\Specna}}\omdot[t-1] + \frac{1-\left(1-\alpha\right)^{m-p_{\se}}}{|\spset|}&&
					\end{flalign}
					\vspace{-0.1in}
				}
				
				\Predict{$\predyt \gets \sign(\sum_{\Specna\in\At} \omega_{t,\Specna}\, \Specna(i_t))$\;}
				
				\Receive{$y_t\in\{-1,1\}$\;}
				
				\Set{$\Ct := \{\Specna \in \spset : \Specna(i_t) = y_t \}$\;}
				
				\eIf(\tcp*[f]{loss update} ){$\predyt\ne y_t$}{
					\vspace{-0.1in}
					\begin{flalign}\label{eq:lossupdate}
					\omdot &\gets
					\begin{cases}
					0 &  \Specna \in \At\cap\bar{\Ct} \\
					\omdot[t-1]& \Specna \not\in \At \\
					\omega_{t,\Specna}\frac{\bom_t(\At)}{\bom_t(\Ct)} & \Specna\in\Ct
					\end{cases}&&
					\end{flalign}

					\ForEach{$\Specna\in\At$}{	
						$p_{\Specna} \gets m$\;
					}
					$m \gets m + 1$\;
				}{
					$\bomdot \gets \bomdot[t-1]$\;
				}	
			}
			\caption{{\sc Switching Cluster Specialists}}\label{Main_Alg}
		\end{small}
	\end{algorithm2e}
	
	In Algorithm~\ref{Main_Alg} we give our switching specialists method.   The algorithm maintains a weight vector $\bom_t$ over the specialists in which the magnitudes may be interpreted as the current confidence we have in each of the specialists.  The updates and their analyses are a combination of three standard methods:  i) {\em Halving} loss updates, ii) specialists updates and iii) (delayed) fixed-share 
	updates.
	The loss update~\eqref{eq:lossupdate} zeros the weight components of incorrectly predicting specialists, while the non-predicting specialists are not updated at all.   In~\eqref{ShareUpdateLogN} we give our {\em delayed} fixed-share style update.  
	A standard fixed share update may be written in the following form:  
	\begin{equation}\label{OldFixedShareUpdate}
	\omega_{t,\Specna} = (1-\alpha)\omdot[t-1] + \frac{\alpha}{|\spset |}\,.
	\end{equation}
	Although~\eqref{OldFixedShareUpdate} superficially appears different to~\eqref{ShareUpdateLogN}, in fact these two updates are exactly the same in terms of predictions generated by the algorithm.  This is because~\eqref{ShareUpdateLogN} caches updates until the given specialist is again active.  The purpose of this computationally is that if the active specialists are, for example, logarithmic in size compared to the total specialist pool, we may then achieve an exponential speedup over~\eqref{OldFixedShareUpdate}; which in fact we will exploit.
	
	In the following theorem we will give our switching specialists bound.  The  dominant cost of switching on trial $t$ to $t+1$ is given by the non-symmetric $\J(\bmu_{t},\bmu_{t+1}) := |\{\Specna\in\spset: \mu_{t,\Specna}= 0, \mu_{t+1,\Specna}\neq 0\}|$, i.e., we pay only for each new specialist introduced but we do not pay for removing specialists.
	\begin{theorem}\label{MainTheorem}
		For a given specialist set $\spset$, let $M_{\spset}$ denote the number of mistakes made in predicting the online sequence $(i_1,y_1),\ldots,(i_{T},y_{T})$ by Algorithm~\ref{Main_Alg}. Then, 
		\begin{equation}\label{TheoremBasicBound}
		M_{\spset}\leq \frac{1}{\pi_1}\log{|\spset|} + \sum_{t=1}^{T}\frac{1}{\pi_t}\log{\frac{1}{1-\alpha}}	+ \sum_{i=1}^{|\sset|-1}\J\!\left(\bmu_{k_i},\bmu_{k_{i+1}}\right)\log{\frac{|\spset|}{\alpha}}\,,
		\end{equation}
		for any sequence of {\bf consistent} and {\bf well-formed} comparators $\bmu_1,\ldots,\bmu_T \in \Delta_{|\spset|}$ 
		where $\sset := \{k_1=1\!<\cdots<k_{|\sset|}\}\!:=\! \{t\!\in\! [T] \!:\! \mut \neq \mut[t-1]\}\cup\{1\}$, and $\pi_t :=\bmu_t(\Ct)$.
	\end{theorem}
	The bound in the above theorem depends crucially on the best sequence of {\em consistent} and {\em well-formed} comparators $\bmu_1,\ldots,\bmu_T$.  The consistency requirement implies that on every trial there is no active incorrect specialist assigned ``mass'' ($\bmu_t(\At\setminus\Ct) = 0$). We may eliminate the consistency requirement by  ``softening''  the loss update~\eqref{eq:lossupdate}. 
	A comparator $\bmu\in\Delta_{|\spset|}$ is \textit{well-formed} if $\forall\ v\in V$, there exists a {\em unique}  $\Specna\in\spset$ such that $\Specna(v)\neq\square$ and $\mu_{\Specna}>0$, and furthermore there exists a $\pi\in (0,1]$ such that $\forall \Specna\in\spset : \mu_{\Specna}\in\{0,\pi\}$, i.e., each specialist in the support  of $\bmu$ has the same mass $\pi$ and these specialists disjointly cover the input space ($V$).  At considerable complication to the form of the bound the well-formedness requirement may be eliminated.

	The above bound is ``smooth'' in that it scales with a gradual change in the comparator.  In the next section we describe the novel specialists sets that we've tailored to graph-label prediction so that a small change in comparator corresponds to a small change in a graph labeling. 
	\subsection{Cluster Specialists}\label{sec:cluster-specialists}
	In order to construct the {\em cluster specialists} over a graph $\cG=(V=[n],E)$, we first construct a line graph as described in Section~\ref{subsec:random-spanning-trees-and-linearization}.  A cluster specialist is then defined by $\Spec{l}{r}{y}(\cdot)$ which maps $V \into \{-1,1,\square\}$ where $\Spec{l}{r}{y}(v) := y$ if $l\le v \le r$ and $\Spec{l}{r}{y}(v) := \square$ otherwise.
	Hence cluster specialist $\Spec{l}{r}{y}(v)$ corresponds to a function that predicts the label $y$ if vertex $v$ lies between vertices $l$ and $r$ and abstains otherwise. 
	Recall that by sampling a random spanning tree the expected cut-size of a labeling on the spine is no more than twice the resistance-weighted cut-size on $\cG$.
	Thus, given a labeled graph with a small resistance-weighted cut-size with densely interconnected clusters and modest intra-cluster connections, this implies a cut-bracketed linear segment on the spine will in expectation roughly correspond to one of the original dense clusters.   We will consider two basis sets of cluster specialists.  

	\paragraph{Basis $\FB_{n}$:}
	We first introduce the {\em complete} basis set $\FB_n := \{\Spec{l}{r}{y} : l,r \in [n] , l\le r ;   y \in\{-1,1\} \}$.
	We say that a set of specialists $\cC_{\bu} \subseteq \spset\subseteq 2^{{\{-1,1,\square\}}^n}$ from basis $\spset$ {\em covers} a labeling $\bu\in\{-1,1\}^n$ if for all $v\in V=[n]$ and $\se\in\cC_{\bu}$ that $\se(v)\in \{u_v,\zzz\}$ and if $v\in V$ then there exists $\se\in\cC_{\bu}$ such that $\se(v) = u_v$.  The basis $\spset$  is {\em complete} if every labeling $\bu\in\{-1,1\}^n$ is covered by some $\cC_{\bu}\subseteq\spset$.  The basis $\FB_n$ is complete and in fact has the following approximation property: for any $\bu\in\{-1,1\}^n$  there exists a covering set $\cC_{\bu} \subseteq \FB_{n}$ such that $|\cC_{\bu}| = \cutS+1$.  This follows directly as a line with $k-1$ cuts is divided into $k$ segments.  We now illustrate the use of basis $\FB_n$ to predict the labeling of a graph.  For simplicity we illustrate by considering the problem of predicting a single graph labeling without switching.    As there is no switch we will set $\Sp :=0$ and thus if the graph is labeled with $\bu\in\{-1,1\}^n$ with cut-size $ \cutSpine(\bu)$ then we will need $\cutSpine(\bu)+1$ specialists to predict the labeling and thus the comparators may be post-hoc optimally determined so that $\bmu = \bmu_1 = \cdots =\bmu_T$ and there will be $ \cutSpine(\bu)+1$ components of $\bmu$ each with ``weight'' $\nicefrac{1}{(\cutSpine(\bu)+1)}$, thus $\nicefrac{1}{\pi_1} =\cutSpine(\bu)+1$,  since there will be only one specialist (with non-zero weight) active per trial. Since the cardinality of $\FB_n$ is $n^2 +n$, by substituting into~\eqref{TheoremBasicBound} we have that the number of mistakes will be bounded by $(\cutSpine(\bu)+1) \log{(n^2+n)}$.  Note for a single graph labeling on a spine this bound is not much worse than the best known result~{\citep[Theorem 4]{\largediam}}. 
	In terms of computation time however it is significantly slower than the algorithm in~\citep{\largediam} requiring $\Theta(n^2)$ time to predict on a typical trial since on average there are $\Theta(n^2)$ specialists active per trial.  
	
	\paragraph{Basis $\BTB_{1,n}$:}
	We now introduce the basis $\BTB_n$ which has $\Theta(n)$ specialists and only requires $\cO(\log n)$ time per trial to predict with only a small increase in bound.
	The basis is defined as
	\begin{equation*}\label{eq:hbase}
	\BTB_{p,q} :=\begin{cases}
	\{\Spec{p}{q}{-1},\Spec{p}{q}{1}\} & p=q, \\
	\{\Spec{p}{q}{-1},\Spec{p}{q}{1}\} \!\cup\! \BTB_{p,\lfloor\frac{p+q}{2}\rfloor} 
	\cup\BTB_{\lfloor\frac{p+q}{2}\rfloor + 1,q} & p\ne q
	\end{cases}
	\end{equation*}
	and is analogous to a binary tree.  We have the following approximation property for $\BTB_n := \BTB_{1,n}$, 
	\begin{proposition}\label{pro:hbasis}   
		The basis $\BTB_{n}$ is complete. Furthermore, for any labeling $\bu\in\{-1,1\}^n$  there exists a covering set $\cC_{\bu} \subseteq \BTB_{n}$ such that $|\cC_{\bu}| \le 2 (\cutS+1) \lceil\log_2 \frac{n}{2}\rceil$ for $n>2$. 
	\end{proposition}

	From a computational perspective the binary tree structure ensures that there are only $\Theta(\log n)$ specialists active per trial,  leading to an exponential speed-up in prediction.  
	A similar set of specialists were used for obtaining adaptive-regret bounds in~\citep{DGS15,KOWW17}.  In that work however the ``binary tree'' structure is over the time dimension (trial sequence) whereas in this work the binary tree is over the space dimension (graph) and a fixed-share update is used to obtain adaptivity over the time dimension.\footnote{An interesting open problem is to try to find good bounds and time-complexity with sets of specialists over {\em both} the time and space dimensions.}

	In the corollary that follows we will exploit the fact that by making the algorithm \textit{conservative} we may reduce the usual $\log{T}$ term in the mistake bound induced by a fixed-share update to $\log{\log{T}}$. A conservative algorithm only updates the specialists' weights on trials on which a mistake is made. 
	Furthermore the bound given in the following corollary is \textit{smooth} 
	as the cost per switch will be measured with a Hamming-like divergence $H$ on the ``cut'' edges between successive labelings,  
	defined as
	\[ H(\bu,\bu') := \!\sum\limits_{(i,j)\in E_{\spine}} \![\ [[u_i \neq u_j]\lor [u_i' \neq u_j' ]]\  \land \ 
	[[u_i \neq u_i'] \lor [u_j \neq u_j' ]]\ ]\,.
	\]
	Observe that $H(\bu,\bu')$ is  smaller than twice the hamming distance between $\bu$ and $\bu'$ and is often significantly smaller.	
	To achieve the bounds we will need the following proposition, which upper bounds divergence $J$ by $H$,
	a subtlety is that there are many distinct sets of specialists consistent with a given comparator.  For example, consider a uniform labeling on $\spine$. One may ``cover'' this labeling with a single specialist or alternatively $n$ specialists, one covering each vertex.   For the sake of simplicity in bounds we will always choose the smallest set of covering specialists.  Thus we introduce the following notions of {\em consistency} and {\em minimal-consistency}.  
	
	\begin{definition}
		A comparator $\bmu \in \Delta_{|\spset|}$ is consistent with the labeling $\bu\in\{-1,1\}^n$ if $\bmu$ is well-formed and $\mu_\Specna >0$ implies that for all $v\in V$ that $\Specna(v)\in\{u_v,\square\}$.
	\end{definition}
	\begin{definition}
		A comparator $\bmu \in \Delta_{|\spset|}$ is minimal-consistent with the labeling $\bu\in\{-1,1\}^n$ if it is consistent with $\bu$ and the cardinality of its support set $|\{\mu_\Specna : \mu_\Specna >0\}|$ is the minimum of all comparators consistent with $\bu$.
	\end{definition}

	\begin{proposition}\label{prop:hamming-bound}
		For a linearized graph $\spine$, for comparators $\bmu, \bmu'\in\Delta_{|\FB_n|}$ that are minimal-consistent with
		$\bu$ and $\bu'$ respectively,
		\begin{equation*}
		\J[\FB_n]\!\!\left(\bmu,\bmu'\right) \leq \min{\left(2H\!\!\left(\bu,\bu'\right)\!, {\cut_\spine}\!\left(\bu'\right)+1\right)}\,.
		\end{equation*}
	\end{proposition}

	A proof is given in Appendix~\ref{sec:hamming-bound-prop-proof}.
	In the following corollary we summarize the results of the \SCS\ algorithm using the basis sets $\FB_n$ and $\BTB_n$ with an optimally-tuned switching parameter $\alpha$. 
	\begin{corollary}\label{co:ss} 
		For a connected $n$-vertex graph $\cG$ and with randomly sampled spine $\spine$, the number of mistakes made in predicting the online sequence $(i_1,y_1),\ldots,(i_{T},y_{T})$ by the \SCS\ algorithm with optimally-tuned $\alpha$ is upper bounded with basis $\FB_n$ by
		\begin{align*}
		&	\mathcal{O}\left( \cut_1\log{n} + \sum_{i=1}^{|\sset| - 1}H(\bu_{k_i},\bu_{k_{i+1}})\left(\log{n}+ \log{|\sset|} + \log{\log{T}}\right)\right)\,
		\end{align*}
		and with basis $\BTB_n$ by 
		\begin{align*}
		&	\mathcal{O}\left(\left(\cut_1\log{n}  + \sum_{i=i}^{|\sset| - 1}H(\bu_{k_i},\bu_{k_{i+1}})\left(\log{n}+ \log{|\sset|} + \log{\log{T}}\right)\right)\log{n}\right)\,
		\end{align*}
		for any sequence of labelings $\bu_1,\ldots,\bu_T \in \{-1,1\}^n$ such that  $u_{t,i_t} = y_t$ for all $t\in [T]$.
	\end{corollary}
	Thus the bounds are equivalent up to a factor of $\log n$ although the computation times vary dramatically.
	See Appendix~\ref{sec:maincor} for a technical proof of these results, and details on the selection of the switching parameter $\alpha$. Note that we may avoid the issue of needing to optimally tune $\alpha$ using the following method  proposed by~\cite{Herbster-TBE2} and by~\cite{koolenswitch}. We use a time-varying parameter and on trial $t$ we set $\alpha_t = \nicefrac{1}{t+1}$. We have the following guarantee for this method, see Appendix~\ref{sec:alpha-t-proof} for a proof.
	\begin{proposition}\label{co:alpha-t-bounds}
		For a connected $n$-vertex graph $\cG$ and with randomly sampled spine $\spine$, the \SCS\ algorithm with bases $\FB_n$ and $\BTB_n$ in predicting the online sequence $(i_1,y_1),\ldots,(i_{T},y_{T})$ now with time-varying $\alpha$ set equal to $\nicefrac{1}{t+1}$ on trial $t$ achieves the same asymptotic  mistake bounds as in Corollary~\ref{co:ss} with an optimally-tuned $\alpha$, under the assumption that $\cut_{\spine}(\bu_1) \leq \sum_{i=1}^{|\sset|-1}\J\!(\bmu_{k_i},\bmu_{k_{i+1}})$.
	\end{proposition}

	\newcommand{\niterations}{$25$}  
	\section{Experiments}\label{sec:Experiments}
	In this section we present results of experiments on real data. The City of Chicago currently contains $608$ public bicycle stations for its ``Divvy Bike'' sharing system. Current and historical data is available  from the City of Chicago\footnote{\tiny{\url{https://data.cityofchicago.org/Transportation/Divvy-Bicycle-Stations-Historical/eq45-8inv}}} containing a variety of features for each station, including latitude, longitude, number of docks, number of operational docks, and number of docks occupied. The latest data on each station is published approximately every ten minutes.

	We used a sample of $72$ hours of data, consisting of three consecutive weekdays in April $2019$. The first $24$ hours of data were used for parameter selection, and the remaining $48$ hours of data were used for evaluating performance.  
	On each ten-minute snapshot we took the percentage of empty docks of each station. We created a binary labeling from this data by setting a threshold of $50\%$. Thus each bicycle station is a vertex in our graph and the label of each vertex indicates whether that station is `mostly full' or `mostly empty'.
	Due to this thresholding the labels of some `quieter' stations were observed not to switch, as the percentage of available docks rarely changed. These stations tended to be on the `outskirts', and thus we excluded these stations from our experiments, giving $404$ vertices in our graph. 

	Using the geodesic distance between each station's latitude and longitudinal position a connected graph was built using the union of a $k$-nearest neighbor graph ($k=3$) and a minimum spanning tree. For each instance of our algorithm the graph was then transformed in the manner described in Section~\ref{subsec:random-spanning-trees-and-linearization}, by first drawing a spanning tree uniformly at random and then linearizing using depth-first search.

	As natural benchmarks for this setting we considered the following four methods.
	$1.)$ For all vertices predict with the most frequently occurring label of the entire graph from the training data (``Global'').
	$2.)$ For each vertex predict with its most frequently occurring label from the training data (``Local'').
	$3.)$ For all vertices at any given time predict with the most frequently occurring label of the entire graph at that time from the training data (``Temporal-Global'')
	$4.)$ For each vertex at any given time predict with that vertex's label observed at the same time in the training data (``Temporal-Local''). 
	We also compare our algorithms against a kernel Perceptron proposed by~\cite{Herbster-SwitchingGraphs} for predicting switching graph labelings (see Appendix~\ref{sec:perceptron-appendix} for details).
	
	Following the experiments of~\cite{CGVZ10} in which ensembles of random spanning trees were drawn and aggregated by an unweighted majority vote, we tested the effect on performance of using ensembles of instances of our algorithms, aggregated in the same fashion. We tested ensemble sizes in $\{1,3,5,9,17,33,65\}$, using odd numbers to avoid ties. 

	For every ten-minute snapshot (labeling) we queried $30$ vertices uniformly at random (with replacement) in an online fashion, giving a sequence of $8640$ trials over $48$ hours. The average performance over \niterations\ iterations is shown in Figure~\ref{fig:experiment-plot}. There are several surprising observations to be made from our results. Firstly, both \SCS\ algorithms performed significantly better than all benchmarks and competing algorithms. Additionally basis $\BTB_{n}$ outperformed basis $\FB_{n}$ by quite a large margin, despite having the weaker bound and being exponentially faster. 
	Finally we observed a significant increase in performance of both \SCS\ algorithms by increasing the ensemble size (see Figure~\ref{fig:experiment-plot}), additional details on these experiments and results of all ensemble sizes are given in Appendix~\ref{sec:appendix-experiment-details}.
	
	
	Interestingly when tuning $\alpha$ we found basis $\BTB_{n}$ to be very robust, while $\FB_{n}$ was very sensitive. This observation combined with the logarithmic per-trial time complexity suggests that \SCS\ with $\BTB_{n}$ has promise to be a very practical algorithm. 
	
	\begin{figure}[tbp]
		\begin{subfigure}{0.55\linewidth}
			\centering
			\includegraphics[width=\linewidth]{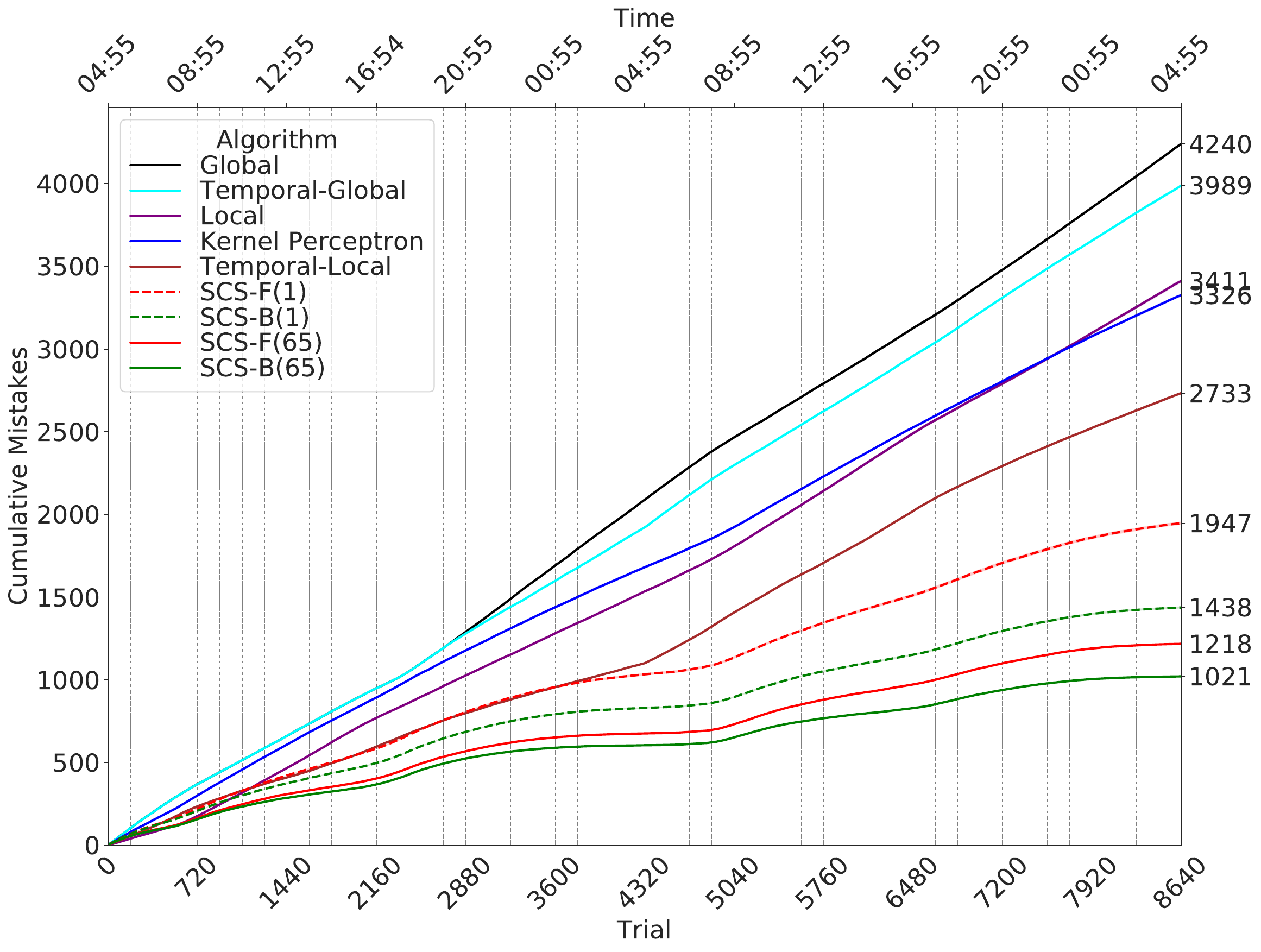}
		\end{subfigure}
		\begin{subfigure}{0.44\linewidth}
			\begin{subfigure}{0.49\linewidth}
				\includegraphics[clip, trim=5cm 18.2cm 9.7cm 1.4cm,width=\linewidth]{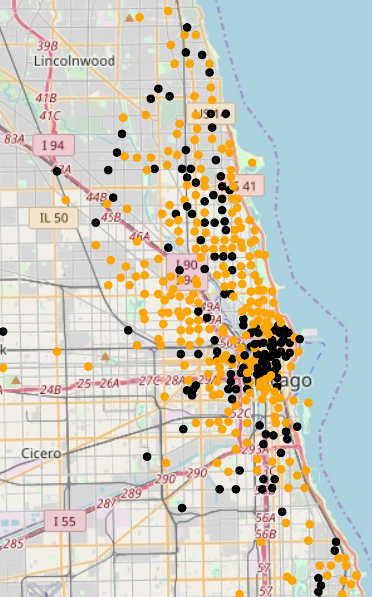}
			\end{subfigure}
			\begin{subfigure}{0.49\linewidth}
				\includegraphics[clip, trim=5cm 18.2cm 9.7cm 1.4cm,width=\linewidth]{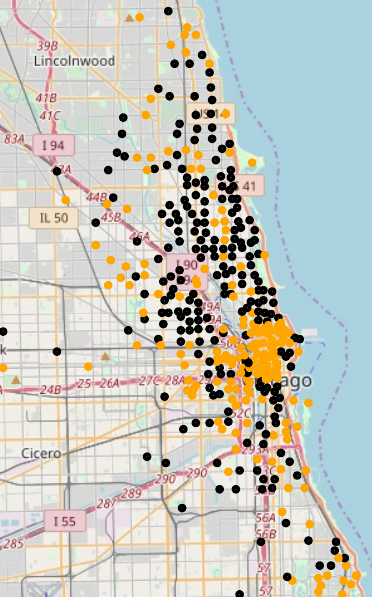}
			\end{subfigure}
		\end{subfigure}
		\vspace*{-.3cm}
		\caption{\footnotesize Left: Mean cumulative mistakes over \niterations\ iterations for all algorithms and benchmarks over $48$ hours ($8640$ trials) on a $404$-vertex graph. A comparison of the mean performance of \SCS\ with bases $\FB_{n}$ and $\BTB_{n}$ (\SCS-F and \SCS-B respectively) using an ensemble of size $1$ and $65$ is shown. Right: An example of two binary labelings taken from the morning and evening of the first $24$ hours of data. An `orange' label implies that station is $<50\%$ full and a `black' label implies that station is $\geq50\%$ full.}
		\label{fig:experiment-plot}
	\end{figure}
	
	\section{Conclusion}\label{sec:conclusion}
	Our primary result was an algorithm for predicting switching graph labelings with a per-trial prediction time of 
	$O(\log n)$ and a mistake bound that {\em smoothly} tracks changes to the graph labeling over time.  In the long version of this paper we plan to extend the analysis of the primary algorithm to the expected regret setting
	by relaxing our simplifying assumption of the {\em well-formed} comparator sequence that is {\em minimal-consistent} with the labeling sequence.
	From a technical perspective the open problem that we found most intriguing is to eliminate the $\log\log T$ term from our bounds.  The natural approach to this would be to replace the conservative {\em fixed-share} update with a  {\em variable-share} update~\citep{Herbster-TrackingExpert}; in our efforts however we found many technical problems with this approach.  On both the more practical and speculative side; we observe that the specialists sets
	$\BTB_n$, and $\FB_n$ were chosen to ``prove bounds''.  In practice we can use any {\em hierarchical} graph clustering algorithm to produce a {\em complete} specialist set and furthermore multiple such clusterings may be pooled.  Such a pooled set of subgraph ``motifs'' could be then be used for example in a multi-task setting (see for example,~\citep{kaw12}).

	\bibliography{SwitchingGraphSpecialists-5.bib}

\begin{thebibliography}{10}

\bibitem{AKCV12}
D.~Adamskiy, W.~M. Koolen, A.~Chernov, and V.~Vovk.
\newblock A closer look at adaptive regret.
\newblock In {\em Proceedings of the 23rd International Conference on
  Algorithmic Learning Theory}, ALT'12, pages 290--304, 2012.

\bibitem{BN04}
M.~Belkin and P.~Niyogi.
\newblock Semi-supervised learning on riemannian manifolds.
\newblock {\em Machine learning}, 56(1-3):209--239, 2004.

\bibitem{Blum:2001}
A.~Blum and S.~Chawla.
\newblock Learning from labeled and unlabeled data using graph mincuts.
\newblock In {\em Proceedings of the Eighteenth International Conference on
  Machine Learning}, ICML '01, pages 19--26, 2001.

\bibitem{Bousquet-MPP}
O.~Bousquet and M.~K. Warmuth.
\newblock Tracking a small set of experts by mixing past posteriors.
\newblock {\em Journal of Machine Learning Research}, 3(Nov):363--396, 2002.

\bibitem{fixmir}
N.~Cesa-Bianchi, P.~Gaillard, G.~Lugosi, and G.~Stoltz.
\newblock Mirror descent meets fixed share (and feels no regret).
\newblock In {\em Proceedings of the 25th International Conference on Neural
  Information Processing Systems - Volume 1}, NIPS '12, pages 980--988, 2012.

\bibitem{CesaBianchi-FastOptimalTrees}
N.~Cesa-Bianchi, C.~Gentile, and F.~Vitale.
\newblock Fast and optimal prediction on a labeled tree.
\newblock In {\em Proceedings of the 22nd Annual Conference on Learning
  Theory}, pages 145--156. Omnipress, 2009.

\bibitem{CGVZ10}
N.~Cesa-Bianchi, C.~Gentile, F.~Vitale, and G.~Zappella.
\newblock Random spanning trees and the prediction of weighted graphs.
\newblock {\em Journal of Machine Learning Research}, 14(1):1251--1284, 2013.

\bibitem{nicolobook}
N.~Cesa-Bianchi and G.~Lugosi.
\newblock {\em Prediction, Learning, and Games}.
\newblock Cambridge University Press, New York, NY, USA, 2006.

\bibitem{DGS15}
A.~Daniely, A.~Gonen, and S.~Shalev-Shwartz.
\newblock Strongly adaptive online learning.
\newblock In {\em Proceedings of the 32nd International Conference on
  International Conference on Machine Learning - Volume 37}, ICML'15, pages
  1405--1411, 2015.

\bibitem{freundopen}
Y.~Freund.
\newblock Private communication, 2000.
\newblock Also posted on http://www.learning-theory.org.

\bibitem{Freund-Specialists}
Y.~Freund, R.~E. Schapire, Y.~Singer, and M.~K. Warmuth.
\newblock Using and combining predictors that specialize.
\newblock In {\em Proceedings of the Twenty-ninth Annual ACM Symposium on
  Theory of Computing}, STOC '97, pages 334--343, 1997.

\bibitem{wacky}
A.~Gy\"{o}rgy, T.~Linder, and G.~Lugosi.
\newblock Tracking the best of many experts.
\newblock In {\em Proceedings of the 18th Annual Conference on Learning
  Theory}, COLT '05, pages 204--216, 2005.

\bibitem{HS07}
E.~Hazan and C.~Seshadhri.
\newblock Adaptive algorithms for online decision problems.
\newblock {\em Electronic Colloquium on Computational Complexity {(ECCC)}},
  14(088), 2007.

\bibitem{Herbster-TBE2}
M.~Herbster.
\newblock Tracking the best expert {II}.
\newblock Unpublished manuscript, 1997.

\bibitem{HL09}
M.~Herbster and G.~Lever.
\newblock Predicting the labelling of a graph via minimum
  {\textdollar}p{\textdollar}-seminorm interpolation.
\newblock In {\em {COLT} 2009 - The 22nd Conference on Learning Theory}, 2009.

\bibitem{Herbster-LargeDiameter}
M.~Herbster, G.~Lever, and M.~Pontil.
\newblock Online prediction on large diameter graphs.
\newblock In {\em Proceedings of the 21st International Conference on Neural
  Information Processing Systems}, NIPS '08, pages 649--656, 2008.

\bibitem{HPG15}
M.~Herbster, S.~Pasteris, and S.~Ghosh.
\newblock Online prediction at the limit of zero temperature.
\newblock In {\em Proceedings of the 28th International Conference on Neural
  Information Processing Systems - Volume 2}, NIPS'15, pages 2935--2943, 2015.

\bibitem{Herbster-SwitchingGraphs}
M.~Herbster, S.~Pasteris, and M.~Pontil.
\newblock Predicting a switching sequence of graph labelings.
\newblock {\em Journal of Machine Learning Research}, 16(1):2003--2022, 2015.

\bibitem{Herbster-GraphPerceptron}
M.~Herbster and M.~Pontil.
\newblock Prediction on a graph with a perceptron.
\newblock In {\em Proceedings of the 19th International Conference on Neural
  Information Processing Systems}, NIPS'06, pages 577--584, 2006.

\bibitem{Herbster-OnlineLearningGraphs}
M.~Herbster, M.~Pontil, and L.~Wainer.
\newblock Online learning over graphs.
\newblock In {\em Proceedings of the 22nd International Conference on Machine
  Learning}, ICML '05, pages 305--312, 2005.

\bibitem{Herbster-TrackingExpert}
M.~Herbster and M.~Warmuth.
\newblock Tracking the best expert.
\newblock {\em Machine Learning}, 32(2):151--178, 1998.

\bibitem{Herbster-TrackLinPred}
M.~Herbster and M.~K. Warmuth.
\newblock Tracking the best linear predictor.
\newblock {\em Journal of Machine Learning Research}, 1:281--309, Sept. 2001.

\bibitem{KOWW17}
K.~Jun, F.~Orabona, S.~Wright, and R.~Willett.
\newblock {Improved strongly adaptive online learning using coin betting}.
\newblock In {\em Proceedings of the 20th International Conference on
  Artificial Intelligence and Statistics}, volume~54 of {\em Proceedings of
  Machine Learning Research}, pages 943--951. PMLR, 20--22 Apr 2017.

\bibitem{K98}
B.~S. Kerner.
\newblock Experimental features of self-organization in traffic flow.
\newblock {\em Phys. Rev. Lett.}, 81:3797--3800, Oct 1998.

\bibitem{KSW04}
J.~Kivinen, A.~Smola, and R.~Williamson.
\newblock Online learning with kernels.
\newblock {\em Trans. Sig. Proc.}, 52(8):2165--2176, Aug. 2004.

\bibitem{Klein-ResistanceDistance}
D.~J. Klein and M.~Randi{\'c}.
\newblock Resistance distance.
\newblock {\em Journal of mathematical chemistry}, 12(1):81--95, 1993.

\bibitem{kaw12}
W.~M. Koolen, D.~Adamskiy, and M.~K. Warmuth.
\newblock Putting bayes to sleep.
\newblock In {\em Proceedings of the 25th International Conference on Neural
  Information Processing Systems - Volume 1}, NIPS '12, pages 135--143, 2012.

\bibitem{koolenswitch}
W.~M. Koolen and S.~Rooij.
\newblock Combining expert advice efficiently.
\newblock In {\em 21st Annual Conference on Learning Theory - {COLT} 2008},
  pages 275--286, 2008.

\bibitem{LW94}
N.~Littlestone and M.~K. Warmuth.
\newblock The weighted majority algorithm.
\newblock {\em Information and Computation}, 108(2):212--261, 1994.

\bibitem{LP17}
R.~Lyons and Y.~Peres.
\newblock {\em Probability on Trees and Networks}.
\newblock Cambridge University Press, New York, NY, USA, 1st edition, 2017.

\bibitem{PSST16}
O.~H.~M. Padilla, J.~Sharpnack, J.~G. Scott, and R.~J. Tibshirani.
\newblock The dfs fused lasso: Linear-time denoising over general graphs.
\newblock {\em Journal of Machine Learning Research}, 18(1):1--36, 2018.

\bibitem{RS17}
A.~Rakhlin and K.~Sridharan.
\newblock Efficient online multiclass prediction on graphs via surrogate
  losses.
\newblock In {\em Proceedings of the 20th International Conference on
  Artificial Intelligence and Statistics, {AISTATS} 2017}, pages 1403--1411,
  2017.

\bibitem{VCGZ11}
F.~Vitale, N.~Cesa-Bianchi, C.~Gentile, and G.~Zappella.
\newblock See the tree through the lines: The shazoo algorithm.
\newblock In {\em Advances in Neural Information Processing Systems 23}, pages
  1584--1592, 2011.

\bibitem{AS90}
V.~Vovk.
\newblock Aggregating strategies.
\newblock In {\em Proceedings of the Third Annual Workshop on Computational
  Learning Theory}, COLT '90, pages 371--386, 1990.

\bibitem{derandom}
V.~Vovk.
\newblock Derandomizing stochastic prediction strategies.
\newblock {\em Machine Learning}, 35(3):247--282, 1999.

\bibitem{Wilson-RST}
D.~B. Wilson.
\newblock Generating random spanning trees more quickly than the cover time.
\newblock In {\em Proceedings of the Twenty-eighth Annual ACM Symposium on
  Theory of Computing}, STOC '96, pages 296--303, 1996.

\bibitem{ZhouBLWS03}
D.~Zhou, O.~Bousquet, T.~N. Lal, J.~Weston, and B.~Sch\"{o}lkopf.
\newblock Learning with local and global consistency.
\newblock In {\em Proceedings of the 16th International Conference on Neural
  Information Processing Systems}, NIPS '03, pages 321--328, 2003.

\bibitem{ZhuGL03}
X.~Zhu, Z.~Ghahramani, and J.~D. Lafferty.
\newblock Semi-supervised learning using gaussian fields and harmonic
  functions.
\newblock In {\em Proceedings of the Twentieth International Conference on
  International Conference on Machine Learning}, ICML '03, pages 912--919,
  2003.

\end{thebibliography}
	
	\newpage
	\appendix
	\section*{Appendix}\label{sec:appendix}

	\section{Proof of Theorem~\ref{MainTheorem}}\label{sec:MainTheoremProof-Appendix}
	\begin{proof}
		Recall that the cached share update~\eqref{ShareUpdateLogN} is equivalent to performing~\eqref{OldFixedShareUpdate}. We thus simulate the latter update in our analysis. 
		We first argue the inequality
		\begin{equation}\label{freunds-realizable-reduction}
		\Ive{\predyt\neq\yt} \le \frac{1}{\bmu_t(\Ct)}\left(\RE[\bmu_t]{\bom_{t}} - \RE[\bmu_t]{\bomdot}\right),    
		\end{equation}
		as this is derived by observing that
		\begin{align*}
		\RE[\bmu_t]{\bom_{t}} - \RE[\bmu_t]{\bomdot} &= 
		\sum_{\Specna\in{\spset} }\mu_{t,\Specna}\log{\frac{\omdot}{\omega_{t,\Specna}}}
		\\
		& = \sum_{\Specna\in{\Ct} }\mu_{t,\Specna}\log{\frac{\omdot}{\omega_{t,\Specna}}} \\
		& \ge  \bmu_t(\Ct) \Ive{\predyt\neq\yt}\,,
		\end{align*}
		where the second line follows the fact that $\mu_{t,\Specna}\log{\frac{\omdot}{\omega_{t,\Specna}}} = 0$ if $\se\not\in\Ct$ as either the specialist $\se$ predicts `$\zzz$' and $\omdot = \omega_{t,\Specna}$ or it predicts incorrectly and hence $\mu_{t,\se}=0$. The third line follows as for $\se\in\Ct$,  $\frac{\omdot}{\omega_{t,\Specna}}\ge 2$ if there has been a mistake on trial $t$ and otherwise the ratio is $\ge 1$. Indeed, since Algorithm~\ref{Main_Alg} is conservative, this ratio is exactly $1$ when no mistake is made on trial $t$, thus without loss of generality we will assume the algorithm makes a mistake on every trial.
		
		For clarity we will now use simplified notation and let $\pi_t := \bmutYt$.
		We now prove the following inequalities which we will add to \eqref{freunds-realizable-reduction} to create a telescoping sum of relative entropy terms and entropy terms.
		\begin{align}
		\frac{1}{\pi_t}\left[\RE[\bmu_t]{\bomdot} - \RE[\bmu_t]{\bom_{t+1}}\right] &\geq -	\frac{1}{\pi_t}\log{\frac{1}{1-\alpha}}\label{eq:inequality-1}\,,\\
		\frac{1}{\pi_t}\RE[\bmu_{t}]{\bom_{t+1}} - \frac{1}{\pi_{t+1}}\RE[\bmu_{t+1}]{\bom_{t+1}} &\geq-\frac{1}{\pi_t}\entropy[t] + \frac{1}{\pi_{t+1}}\entropy[t+1] - \J\left(\bmu_t, \bmu_{t+1}\right)\log{\frac{|\spset|}{\alpha}}\label{eq:inequality-2}\,.
		\end{align}
		Firstly \eqref{eq:inequality-1} is proved with the following
		\begin{equation*}
		\RE[\bmu_t]{\bomdot} - \RE[\bmu_t]{\bom_{t+1}}
		= \sum_{\Specna\in{\spset} }\mu_{t,\Specna}\log{\frac{\omega_{t+1,\Specna}}{\omdot}}
		\geq \sum_{\Specna\in{\spset} }\mu_{t,\Specna}\log{\left(\frac{(1-\alpha)\omdot}{\omdot}\right)}
		= \log{\left(1-\alpha\right)}\,,
		\end{equation*}
		where the inequality has used $\omega_{t+1,\Specna} \geq (1-\alpha)\omdot$ from \eqref{OldFixedShareUpdate}.
		
		To prove~\eqref{eq:inequality-2} we first define the following sets.
		\begin{align*}
		\forgetset &:= \{\Specna\in\spset: \mu_{t-1,\Specna}\neq 0, \mu_{t,\Specna} =0 \}\,,\\
		\contset &:= \{\Specna\in\spset: \mu_{t-1,\Specna}\neq 0, \mu_{t,\Specna} \neq 0 \}\,,\\
		\newset &:= \{\Specna\in\spset: \mu_{t-1,\Specna} = 0, \mu_{t,\Specna} \neq 0 \}\,.
		\end{align*}
		We now expand the following 
		\begin{align}\label{eq:long-expansion-d(u,v)}
		&\frac{1}{\pi_t}\RE[\bmu_{t}]{\bom_{t+1}} - \frac{1}{\pi_{t+1}}\RE[\bmu_{t+1}]{\bom_{t+1}}\notag\\
		&\qquad= 	\frac{1}{\pi_t}\RE[\bmu_{t}]{\bom_{t+1}} -\frac{1}{\pi_t}\RE[\bmu_{t+1}]{\bom_{t+1}} + \frac{1}{\pi_{t}}\RE[\bmu_{t+1}]{\bom_{t+1}} - \frac{1}{\pi_{t+1}}\RE[\bmu_{t+1}]{\bom_{t+1}}\notag\\
		&\qquad=  \frac{1}{\pi_t}\sum_{\Specna\in{\spset}}\mu_{t,\Specna}\log{\frac{\mu_{t,\Specna}}{\omega_{t+1,\Specna}}} -  \frac{1}{\pi_t}\sum_{\Specna\in{\spset}}\mu_{t+1,\Specna}\log{\frac{\mu_{t+1,\Specna}}{\omega_{t+1,\Specna}}}\notag\\
		&\qquad\qquad +  \frac{1}{\pi_t}\sum_{\Specna\in{\spset}}\mu_{t+1,\Specna}\log{\frac{\mu_{t+1,\Specna}}{\omega_{t+1,\Specna}}} -  \frac{1}{\pi_{t+1}}\sum_{\Specna\in{\spset}}\mu_{t+1,\Specna}\log{\frac{\mu_{t+1,\Specna}}{\omega_{t+1,\Specna}}}\notag\\
		&\qquad= -\frac{1}{\pi_t}\entropy + 	\frac{1}{\pi_t}\entropy[t+1] +  \sum_{\Specna\in{\spset}}\left(\frac{\mu_{t,\Specna}}{\pi_{t}} - \frac{\mu_{t+1,\Specna}}{\pi_{t}}\right)\log{\frac{1}{\omega_{t+1,\Specna}}}\notag\\
		&\qquad\qquad - \frac{1}{\pi_t}\entropy[t+1] + \frac{1}{\pi_{t+1}}\entropy[t+1] + \sum_{\Specna\in{\spset}}\left(\frac{\mu_{t+1,\Specna}}{\pi_{t}} - \frac{\mu_{t+1,\Specna}}{\pi_{t+1}}\right)\log{\frac{1}{\omega_{t+1,\Specna}}}\,.
		\end{align}
		Recall that a comparator $\bmu\in\Delta_{|\spset|}$ is \textit{well-formed} if $\forall\ v\in V$, there exists a {\em unique}  $\Specna\in\spset$ such that $\Specna(v)\neq\square$ and $\mu_{\Specna}>0$, and furthermore there exists a $\pi\in (0,1]$ such that $\forall \Specna\in\spset : \mu_{\Specna}\in\{0,\pi\}$, i.e., each specialist in the support of $\bmu$ has the same mass $\pi$ and these specialists disjointly cover the input space ($V$). Thus, by collecting terms into the three sets $\forgetset[t+1]$, $\contset[t+1]$, and $\newset[t+1]$ we have
		\begin{align}\label{eq:set-cancel-expansion-1}
		&\sum_{\Specna\in{\spset}}\left(\frac{\mu_{t,\Specna}}{\pi_{t}} - \frac{\mu_{t+1,\Specna}}{\pi_{t}}\right)\log{\frac{1}{\omega_{t+1,\Specna}}}\notag\\
		&\qquad=\sum_{\Specna\in{\forgetset[t+1]}}\frac{\mu_{t,\Specna}}{\pi_t}\log{\frac{1}{\omega_{t+1,\Specna}}}
		+ \sum_{\Specna\in{\contset[t+1]}}\left(\frac{\mu_{t,\Specna}}{\pi_{t}} - \frac{\mu_{t+1,\Specna}}{\pi_{t}}\right)\log{\frac{1}{\omega_{t+1,\Specna}}}
		 -\sum_{\Specna\in{\newset[t+1]}}\frac{\mu_{t+1,\Specna}}{\pi_{t}}\log{\frac{1}{\omega_{t+1,\Specna}}}\notag\\
		&\qquad= \sum_{\Specna\in{\forgetset[t+1]}}\frac{\mu_{t,\Specna}}{\pi_t}\log{\frac{1}{\omega_{t+1,\Specna}}} + \sum_{\Specna\in{\contset[t+1]}}\left(1 - \frac{\mu_{t+1,\Specna}}{\pi_{t}}\right)\log{\frac{1}{\omega_{t+1,\Specna}}}
		-\sum_{\Specna\in{\newset[t+1]}}\frac{\mu_{t+1,\Specna}}{\pi_{t}}\log{\frac{1}{\omega_{t+1,\Specna}}}\,,
		\end{align}
		and similarly
		\begin{align}\label{eq:set-cancel-expansion-2}
		&\sum_{\Specna\in{\spset}}\left(\frac{\mu_{t+1,\Specna}}{\pi_{t}} - \frac{\mu_{t+1,\Specna}}{\pi_{t+1}}\right)\log{\frac{1}{\omega_{t+1,\Specna}}}\notag\\
		&\qquad= \sum_{\Specna\in{\contset[t+1]}}\left(\frac{\mu_{t+1,\Specna}}{\pi_{t}} - 1\right)\log{\frac{1}{\omega_{t+1,\Specna}}} + \sum_{\Specna\in{\newset[t+1]}}\left(\frac{\mu_{t+1,\Specna}}{\pi_{t}} - 1\right)\log{\frac{1}{\omega_{t+1,\Specna}}}\,.
		\end{align}
		Substituting~\eqref{eq:set-cancel-expansion-1} and~\eqref{eq:set-cancel-expansion-2} into~\eqref{eq:long-expansion-d(u,v)} and simplifying gives
		\begin{align}
		&\frac{1}{\pi_t}\RE[\bmu_{t}]{\bom_{t+1}} - \frac{1}{\pi_{t+1}}\RE[\bmu_{t+1}]{\bom_{t+1}}\notag\\
		&\qquad= -\frac{1}{\pi_t}\entropy + \frac{1}{\pi_{t+1}}\entropy[t+1]
		+ \sum_{\Specna\in{\forgetset[t+1]}}\frac{\mu_{t,\Specna}}{\pi_t}\log{\frac{1}{\omega_{t+1,\Specna}}}
		- \sum_{\Specna\in{\newset[t+1]}}\log{\frac{1}{\omega_{t+1,\Specna}}}\notag\\
		&\qquad\geq -\frac{1}{\pi_t}\entropy + \frac{1}{\pi_{t+1}}\entropy[t+1] - \left|\newset[t+1]\right|\log{\frac{|\spset|}{\alpha}}\,,
		\end{align}
		where the inequality has used the fact that $\frac{\alpha}{|\spset|}\leq \omega_{t+1,\Specna}\leq 1$ from \eqref{OldFixedShareUpdate}.
		
		Summing over all trials then leaves a telescoping sum of relative entropy terms, a cost of $\frac{1}{\pi_t}\log{\frac{1}{1-\alpha}}$ on each trial, and $|\newset[t+1]|\log{\frac{|\spset|}{\alpha}}$ for each switch. Thus,
		\begin{align}
		\sum_{t=1}^{T}\Ive{\predyt\neq\yt} &\leq\frac{1}{\pi_1}\RE[\bmu_1]{\bom_1}
		+ \frac{1}{\pi_1}\entropy[1]
		+ \sum_{t=1}^{T}\frac{1}{\pi_t}\log{\frac{1}{1-\alpha}} + \sum_{i=1}^{|\sset|-1}\J\left(\bmu_{k_i}, \bmu_{k_{i+1}}\right)\log{\frac{|\spset|}{\alpha}}\,,
		\end{align}
		where $\J(\bmu_{k_i}, \bmu_{k_{i+1}})=|\newset[k_{i+1}]|$, and since $\bom_1 = \frac{1}{|\spset|}\bm 1$, we can combine the remaining entropy and relative entropy terms to give $\frac{1}{\pi_1}\RE[\bmu_1]{\bom_1} + \frac{1}{\pi_1}\entropy[1] = \frac{1}{\pi_1}\log{|\spset|}$, concluding the proof.
	\end{proof}

	\newcommand{\clusterset}[1][\bm{u}]{\mathcal{L}_{#1}}
	\newcommand{\clustercoveringset}[1][(l,r)]{\BTB_{#1}}
	\newcommand{\cluster}{(l,r)}
	\newcommand{\clustercomplexity}{\delta(\clustercoveringset)}
	\newcommand{\setofclustercoveringsets}{\mathcal{D}^{\cluster}}
	\newcommand{\mincomplexset}{\BTB_{\cluster}^*}
	\newcommand{\optcluster}{(l^*,r^*)}
	\newcommand{\clustercoveringsetHAT}{\hat{\BTB}_{(l,r)}}
	\section{Proof of Proposition~\ref{pro:hbasis}}
	We recall the proposition:
	
	\textit{The basis $\BTB_{n}$ is complete.  Furthermore, for any labeling $\bu\in\{-1,1\}^n$ there exists a covering set $\cC_{\bu} \subseteq \BTB_{n}$ such that $|\cC_{\bu}| \le 2 (\cutS+1) \lceil\log_2 \frac{n}{2}\rceil$. }
	
	We first give a brief intuition of the proof; any required terms will be defined more completely later. For a given labeling $\bu\in\{-1,1\}^n$ of cut-size $\cutS$, the spine $\spine$ can be cut into $\cutS+1$ \textit{clusters}, where a cluster is a contiguous segment of vertices with the same label. We will upper bound the maximum number of cluster specialists required to cover a single cluster, and therefore obtain an upper bound for $|\cC_{\bu}|$ by summing over the $\cutS+1$ clusters. 
	
	Without loss of generality  we assume $n=2^r$ for some integer $r$ and thus the structure of $\BTB_n$ is analogous to a perfect binary tree of depth $d=\log_2{n}$. Indeed, for a fixed label parameter $y$ we will adopt the terminology of binary trees such that for instance we say specialist $\Spec{i}{j}{y}$ for $i\neq j$ has a so-called \textit{left-child} $\Spec{i}{\lfloor \frac{i+j}{2}\rfloor}{y}$ and \textit{right-child} $\Spec{\lceil \frac{i+j}{2}\rceil}{j}{y}$. Similarly, we say that $\Spec{i}{\lfloor \frac{i+j}{2}\rfloor}{y}$ and $\Spec{\lceil \frac{i+j}{2}\rceil}{j}{y}$ are \textit{siblings}, and $\Spec{i}{j}{y}$ is their \textit{parent}. Note that any specialist is both an ancestor and a descendant of itself, and a proper descendant of a specialist is a descendant of one of its children. Finally the \textit{depth} of specialist $\Spec{i}{j}{y}$ is defined to be equal to the depth of the corresponding node in a binary tree, such that $\Spec{1}{n}{y}$ is of depth $0$, $\Spec{1}{\frac{n}{2}}{y}$ and $\Spec{\frac{n}{2}+1}{n}{y}$ are of depth $1$, etc.
	
	The first claim of the proposition is easy to prove as $\{\Spec{i}{i}{-1},\Spec{i}{i}{1} : i\in[n]\}\subset \BTB_n$ and thus any labeling $\bu\in\{-1,1\}^n$ can be covered. We now prove the second claim of the proposition.
	
	We will denote a uniformly-labeled contiguous segment of vertices by the pair $\cluster$, where $l,r\in[n]$ are the two end vertices of the segment. For completeness we will allow the trivial case when $l=r$. 
	Given a labeling $\bu\in\{-1,1\}^n$, let $\clusterset:=\{\cluster : 1\leq l\leq r\leq n; u_l=\ldots =u_r; u_{l-1}\neq u_l; u_{r+1}\neq u_r\}$ be the set of maximum-sized contiguous segments of unifmormly-labeled vertices. Note that $u_{l-1}$ or $u_{r+1}$ may be vacuous. When the context is clear, we will also describe $\cluster$ as a \textit{cluster}, and as the set of vertices $\{l,\ldots,r\}$.
	
	For a given $\bu\in\{-1,1\}^n$ and cluster $\cluster\in\clusterset$, we say $\clustercoveringset\subseteq\BTB_n$ is an $\cluster$-covering set with respect to $\bu$ if for all $\Spec{i}{j}{y}\in\clustercoveringset$ we have $l\leq i, j\leq r$, and if for all $k\in\cluster$ there exists some $\Spec{i}{j}{y}\in\clustercoveringset$ such that $ i\leq k \leq j$ and $y=u_k$. That is, every vertex in the cluster is `covered' by at least one specialist and no specialists cover any vertices $k\notin \cluster$. We define $\setofclustercoveringsets$ to be the set of all possible $\cluster$-covering sets with respect to $\bu$.
	
	We now define
	\begin{equation*}
	\clustercomplexity := |\clustercoveringset|
	\end{equation*}
	to be the \textit{complexity} of $\clustercoveringset\in\setofclustercoveringsets$.
	
	For a given $\bu\in\{-1,1\}^n$ and cluster $\cluster\in\clusterset$, we wish to produce an $\cluster$-covering set of \textit{minimum} complexity, which we denote $\mincomplexset:=\argmin\limits_{\clustercoveringset\in\setofclustercoveringsets}\clustercomplexity$. Note that an $\cluster$-covering set of minimum complexity cannot contain any two specialists which are siblings, since they can be removed from the set and replaced by their parent specialist. 

	\begin{lem}\label{lem:BTBmax2ofsamedepth}
		For any $\bu\in\{-1,1\}^n$, for any $\cluster\in\clusterset$, the $\cluster$-covering set of minimum complexity, $\mincomplexset=\argmin\limits_{\clustercoveringset\in\setofclustercoveringsets}\clustercomplexity$ contains at most two specialists of each unique depth.
	\end{lem}
	\begin{proof}
		We first give an intuitive sketch of the proof. For a given $\bu\in\{-1,1\}^n$ and cluster $\cluster\in\clusterset$ assume that there are at least three specialists of equal depth in $\mincomplexset$, then any of these specialists that are in the `middle' may be removed, along with any of their siblings or proper descendants that are also members of $\mincomplexset$ without creating any `holes' in the covering, decreasing the complexity of $\mincomplexset$.

		We use a proof by contradiction. Suppose for contradiction that for a given $\bu\in\{-1,1\}^n$ and $\cluster\in\clusterset$, the $\cluster$-covering set of minimum complexity, $\mincomplexset$, contains three distinct specialists of the same depth, $\Spec{a}{b}{y}, \Spec{c}{d}{y}, \Spec{e}{f}{y}$. 
		Without loss of generality let $a,b < c,d < e,f$. Note that we have $l\leq a < f \leq r$. We consider the following two possible scenarios: when two of the three specialists are siblings, and when none are.
		
		If $\Spec{a}{b}{y}$ and $\Spec{c}{d}{y}$  are siblings, then we have $\Spec{a}{d}{y}\in\BTB_n$ and thus $\{\Spec{a}{d}{y}\}\cup\mincomplexset\setminus\{\Spec{a}{b}{y}, \Spec{c}{d}{y}\}$ is an $\cluster$-covering set of smaller complexity, leading to a contradiction. The equivalent argument holds if $\Spec{c}{d}{y}$ and $\Spec{e}{f}{y}$  are siblings.
		
		If none are siblings, then let $\Spec{c'}{d'}{y}$ be the sibling of $\Spec{c}{d}{y}$ and let $\Spec{C}{D}{y}$ be the parent of $\Spec{c}{d}{y}$ and $\Spec{c'}{d'}{y}$. Note that $a,b < c',d',c,d$ and $c',d',c,d < e,f$ and hence $l < C<D < r$. If an ancestor of $\Spec{C}{D}{y}$ is in $\mincomplexset$, then  $\mincomplexset\setminus\{\Spec{c}{d}{y}\}$ is an $\cluster$-covering set of smaller complexity, leading to a contradiction.
		Alternatively, if no ancestor of $\Spec{C}{D}{y}$ is in $\mincomplexset$, then $\Spec{c'}{d'}{y}$ or some of its proper descendants must be in $\mincomplexset$, otherwise there exists some vertex $k'\in(c',d')$ such that there exists no specialist $\Spec{i}{j}{y}\in\mincomplexset$ such that $i\leq k' \leq j$, and therefore $\mincomplexset$ would not be an $(l,r)$-covering set. Let $\Spec{p}{q}{y}$ be a descendant of $\Spec{c'}{d'}{y}$ which is contained in $\mincomplexset$. Then $\{\Spec{C}{D}{y}\}\cup\mincomplexset\setminus\{\Spec{c}{d}{y},\Spec{p}{q}{y}\}$ is an $\cluster$-covering set of smaller complexity, leading to a contradiction.
		
		We conclude that there can be no more than $2$ specialists of the same depth in $\mincomplexset$ for any $\bu\in\{-1,1\}^n$ and any $\cluster\in\clusterset$.
	\end{proof}
	We now prove an upper bound on the maximum minimum-complexity of an $\cluster$-covering set under any labeling $\bu$.
	
	\begin{corollary}
		For all $\bu \in\{-1,1\}^n$,    
		\begin{equation}
		\label{eqn:max-min-condition-proposition-proof}
		\max_{\cluster\in \clusterset}\min_{\clustercoveringset\in\setofclustercoveringsets}{\clustercomplexity} \leq 2\log{\frac{n}{2}}\,.
		\end{equation}
	\end{corollary}
	\begin{proof}
		For any $\bu\in\{-1,1\}^n$, and $\cluster\in\clusterset$, since $\mincomplexset$ can contain at most $2$ specialists of the same depth (Lemma~\ref{lem:BTBmax2ofsamedepth})
		an $\cluster$-covering set of minimum-complexity can have at most two specialists of depths $2,3,\ldots,d$. This set cannot contain two specialists of depth $1$ as they are siblings. This upper bounds the maximum minimum-complexity of any $\cluster$-covering set by $2(d-2) = 2\log{\frac{n}{2}}$.
	\end{proof}
	
	Finally we conclude that for any labeling $\bu\in\{-1,1\}^n$ of cut-size $\cutS$, there exists $\cC_{\bu} \subseteq \BTB_{n}$ such that $|\cC_{\bu}| \le 2 \log_2{(\frac{n}{2})}(\cutS+1) $.

	\section{Proof of Proposition~\ref{prop:hamming-bound}}\label{sec:hamming-bound-prop-proof}
	
	First recall the proposition statement.
	\begingroup
	\def\thetheorem{\ref{prop:hamming-bound}}
	\begin{proposition}
		For a linearized graph $\spine$, for comparators $\bmu, \bmu'\in\Delta_{|\FB_n|}$ that are minimal-consistent with
		$\bu$ and $\bu'$ respectively,
		\begin{equation*}
		\J[\FB_n]\!\!\left(\bmu,\bmu'\right) \leq \min{\left(2H\!\!\left(\bu,\bu'\right)\!, {\cut_\spine}\!\left(\bu'\right)+1\right)}\,.
		\end{equation*}
	\end{proposition}
	\addtocounter{thm}{-1}
	\endgroup

	\begin{proof}
		We prove both inequalities separately. We first prove $\J[\FB_n](\bmu,\bmu')\leq \cut_{\spine}(\bu')+1$. This follows directly from the fact that $\J(\bmu,\bmu') := |\{\Specna\in\spset: \mu_{\Specna}= 0, \mu_{\Specna}'\neq 0\}|$ and therefore $\J[\FB_n](\bmu,\bmu')\leq |\{\Specna\in\FB_n : \mu_{\Specna}' \neq 0\}| = \cut_{\spine}(\bu')+1$.
		
		We now prove $\J[\FB_n](\bmu,\bmu')\leq 2H\!\!\left(\bu,\bu'\right)$.
		Recall that if $\bu\neq\bu'$ then by definition of the minimal-consistent comparators $\bmu$ and $\bmu'$, the set $\{\Specna\in\FB_n: \mu_{\Specna}= 0, \mu'_{\Specna}\neq 0\}$ corresponds to the set of maximum-sized contiguous segments of vertices in $\spine$ sharing the same label in the labeling $\bu'$ that did not exist in the labeling $\bu$. From here on we refer to a maximum-sized contiguous segment as just a contiguous segment.
		
		When switching from labeling $\bu$ to $\bu'$, we consider the following three cases. First when a non-cut edge (with respect to $\bu$) becomes a cut edge (with respect to $\bu'$), second when a cut edge (with respect to $\bu$) becomes a non-cut edge (with respect to $\bu'$), and lastly when a cut edge remains a cut edge, but the labeling of the two corresponding vertices are `swapped'.
		
		Formally then, for an edge $(i,j)\in E_\spine$ such that $[u_i = u_j] \land [u_i' \neq u_j']$ there exists two new contiguous segments of vertices sharing the same label that did not exist in the labeling $\bu$, their boundary being the edge $(i,j)$.
		
		Conversely for an edge $(i,j)\in E_\spine$ such that $[u_i \neq u_j] \land [u_i' = u_j' ]$ there exists one new contiguous segment of vertices sharing the same label that did not exist in the labeling $\bu$, that segment will contain the edge $(i,j)$.
		
		Finally for an edge $(i,j)\in E_\spine$ such that $[[u_i \neq u_j] \land [u_i' \neq u_j' ]] \land [[u_i \neq u_i'] \lor [u_j \neq u_j' ]]$ there exists two new contiguous segments of vertices sharing the same label that did not exist in the labeling $\bu$, their boundary being the edge $(i,j)$.
		
		We conclude that the number of new contiguous segments of vertices sharing the same label that did not exist in the labeling $\bu$ is upper bounded by
		\begin{equation*}
		2\sum\limits_{(i,j)\in E_{\spine}}[[u_i \neq u_j] \lor [u_i' \neq u_j' ]] \land [[u_i \neq u_i'] \lor [u_j \neq u_j' ]]\,.		
		\end{equation*}
	\end{proof}

	\section{Proof of Corollary~\ref{co:ss}}\label{sec:maincor}
	First recall the corollary statement.
	
	\begingroup
	\def\thetheorem{\ref{co:ss}}
	\begin{corollary}
		For a connected $n$-vertex graph $\cG$ and with randomly sampled spine $\spine$, the number of mistakes made in predicting the online sequence $(i_1,y_1),\ldots,(i_{T},y_{T})$ by the \SCS\ algorithm with optimally-tuned $\alpha$ is upper bounded with basis $\FB_n$ by
		\begin{align*}
		&	\mathcal{O}\left( \cut_1\log{n} + \sum_{i=1}^{|\sset| - 1}H(\bu_{k_i},\bu_{k_{i+1}})\left(\log{n}+ \log{|\sset|} + \log{\log{T}}\right)\right)\,
		\end{align*}
		and with basis $\BTB_n$ by 
		\begin{align*}
		&	\mathcal{O}\left(\left(\cut_1\log{n}  + \sum_{i=i}^{|\sset| - 1}H(\bu_{k_i},\bu_{k_{i+1}})\left(\log{n}+ \log{|\sset|} + \log{\log{T}}\right)\right)\log{n}\right)\,
		\end{align*}
		for any sequence of labelings $\bu_1,\ldots,\bu_T \in \{-1,1\}^n$ such that  $u_{t,i_t} = y_t$ for all $t\in [T]$.
	\end{corollary}
	\addtocounter{thm}{-1}
	\endgroup
	\begin{proof}
		Since Algorithm~\ref{Main_Alg} has a conservative update, we may ignore trials on which no mistake is made and thus from the point of view of the algorithm a mistake is made on every trial, we will therefore assume that $T=M$. This will lead to a self-referential mistake bound in terms of the number of mistakes made which we will then iteratively substitute into itself. 
		
		Let $c:= \log_2{e}$, we will use the fact that $\log_2{(\frac{1}{1-\frac{x}{y+x}})}\leq c\frac{x}{y}$ for $x, y > 0$. We will first optimally tune $\alpha$ to give our tuned mistake bound for a general basis set $\spset$, and then derive the bounds for bases $\FB_{n}$ and $\BTB_{n}$ respectively.
		The value of $\alpha$ that minimizes~\eqref{TheoremBasicBound} is 
		\begin{equation}\label{optimal-alpha}
		\alpha = \frac{\sum\limits_{i=1}^{|\sset|-1}\J\!\left(\bmu_{k_i},\bmu_{k_{i+1}}\right)}{\sum\limits_{t=1}^{T}\frac{1}{\pi_t} + \sum\limits_{i=1}^{|\sset|-1}\J\!\left(\bmu_{k_i},\bmu_{k_{i+1}}\right)}\,,
		\end{equation}
		which when substituted into the second term of~\eqref{TheoremBasicBound} gives
		\begin{equation}\label{TheoremBasicBound-alpha-first-sub}
		M_{\spset}\leq \frac{1}{\pi_1}\log{|\spset|} + c\sum_{i=1}^{|\sset|-1}\J\!\left(\bmu_{k_i},\bmu_{k_{i+1}}\right)	+ \sum_{i=1}^{|\sset|-1}\J\!\left(\bmu_{k_i},\bmu_{k_{i+1}}\right)\log{\frac{|\spset|}{\alpha}}\,.
		\end{equation}	
		We now upper bound $\frac{1}{\alpha}$ for substitution in the last term of~\eqref{TheoremBasicBound-alpha-first-sub} for bases $\FB_n$ and $\BTB_n$ separately.
		
		\paragraph{Basis $\FB_n$ :}For $\FB_n$ observe that $|\spset|=n^2 + n$, and since any labeling $\bu_t\in\{-1,1\}^n$ of cut-size $\cutS[\bu_t]$ is covered by $\cutS[\bu_t]+1$ specialists, we have that $\pi_t = 1/(\cutS[\bu_t] + 1)$ on all trials. We let the number of mistakes made by \SCS\ with basis $\FB_{n}$ be denoted by $M_{\FB_n}$. Thus~\eqref{TheoremBasicBound-alpha-first-sub} immediately becomes
		\begin{equation}\label{TheoremBasicBound-alpha-sub-Fn}
		M_{\FB_n} \leq \left(\cut_1 + 1\right)\log{|\FB_n|}
		+ c\sum_{i=1}^{|\sset|-1}\!\!\J[\FB_n]\!\!\left(\bmu_{k_i},\bmu_{k_{i+1}}\right)
		+\sum_{i=1}^{|\sset|-1}\!\!\J[\FB_n]\!\!\left(\bmu_{k_i},\bmu_{k_{i+1}}\right)\log{\frac{|\FB_n|}{\alpha}}\,.
		\end{equation}
		
		To upper bound $\frac{1}{\alpha}$ we note that if $\bmu_{k_i}\neq \bmu_{k_{i+1}}$ then $\J[\FB_n]\!\!\left(\bmu_{k_i},\bmu_{k_{i+1}}\right)\geq 1$, and that for $\FB_n$, $\frac{1}{\pi_i} =\cut_{k_i} + 1 \leq n$, thus from~\eqref{optimal-alpha} we have
		\begin{align*}
		\frac{1}{\alpha} &= 1 + \frac{\sum\limits_{t=1}^{T}\frac{1}{\pi_t}}
		{\sum\limits_{i=1}^{|\sset|-1}\J[\FB_n]\!\!\left(\bmu_{k_i},\bmu_{k_{i+1}}\right)} \leq 1 + \frac{nT}{|\sset|-1} \leq \frac{nT + |\sset|-1}{|\sset|-1} \leq\frac{(n+1)T}{|\sset|-1}\,.
		\end{align*}
		Substituting $\frac{1}{\alpha}\leq\frac{(n+1)T}{|\sset|-1}$ into~\eqref{TheoremBasicBound-alpha-sub-Fn} gives
		\begin{align}\label{eq:FB-tuned-bound-1}
		M_{\FB_n} &\leq \left(\cut_1 + 1\right)\log{|\FB_n|}+\!\!\sum_{i=1}^{|\sset|-1}\!\!\J[\FB_n]\!\!\left(\bmu_{k_i},\bmu_{k_{i+1}}\right)\left[\log{\left(e|\FB_n|\right)} + \log{\left(n+1\right)} + \log{\frac{T}{|\sset|-1}}\right]
		\end{align}
		We now show our method to reduce the $\log{T}$ term in our bound to $\log{\log{T}}$ by substituting the self-referential mistake bound into itself. We first simplify~\eqref{eq:FB-tuned-bound-1} and substitute $T=M_{\FB_n}$,
		\newcommand{\jj}{\mathcal{J}}
		\newcommand{\ZZ}{\mathcal{Z}}
		\begin{align}
		M_{\FB_n} &\leq \underbrace{\left(\cut_1 + 1\right)\log{|\FB_n|} +\sum_{i=1}^{|\sset|-1}\J[\FB_n]\!\!\left(\bmu_{k_i},\bmu_{k_{i+1}}\right)\log{\left(\frac{e|\FB_n|(n+1)}{|\sset|-1}\right)}}_{=: \ZZ }\notag\\
		&\qquad + \underbrace{\sum_{i=1}^{|\sset|-1}\J[\FB_n]\!\!\left(\bmu_{k_i},\bmu_{k_{i+1}}\right)}_{=:\jj}\log{M_{\FB_n}}\notag\\
		&\leq \ZZ + \jj\log{\left(\ZZ+ \jj\log{M_{\FB_n}}\right)}\notag\\
		&\leq \ZZ + \jj\log{\ZZ} + \jj\log{\jj} +  \jj\log{\log{M_{\FB_n}}}\notag\,,
		\end{align}
		using $\log{(a+b)} \leq \log{(a)} + \log{(b)} \text{ for } a,b \geq 2$.
		We finally use the fact that $\jj=\cO(n|\sset|)$ to give $\jj\log{\jj} = \cO(\jj\log{(n|\sset|)})$ and similarly
		\begin{align*}
		\jj\log{\ZZ} &= \cO(\jj\log{(\cut_1\log{n} + \jj\log{n})})\\
		&= \cO(\jj\log{((n + \jj)\log{n}))})\\
		&=\cO(\jj\log{(n+\jj)})\\
		&= \cO(\jj\log{(n|\sset|)})\,,
		\end{align*}
		to give 
		\begin{equation*}
		M_{\FB_n} \leq \cO \left(\cut_1 \log{n} + \sum_{i=1}^{|\sset|-1}\J[\FB_n]\!\!\left(\bmu_{k_i},\bmu_{k_{i+1}}\right)\left(\log{n} + \log{|\sset|}+ \log{\log{T}}\right)\right)\,.
		\end{equation*}

		\paragraph{Basis $\BTB_n$:} For $\BTB_n$ we apply the same technique as above, but first observe the following. Without loss of generality assume $n=2^r$ for some integer $r$, we then have $|\spset| = 4n - 2$. We let the number of mistakes made by \SCS\ with basis $\BTB_{n}$ be denoted by $M_{\BTB_n}$. Thus for basis $\BTB_n$~\eqref{TheoremBasicBound-alpha-first-sub} becomes
		\begin{align}\label{TheoremBasicBound-alpha-sub-Bn}
		M_{\BTB_n} &\leq 2\log{\frac{n}{2}}\left(\cut_1 + 1\right)\log{|\BTB_n|}
		+ c\sum_{i=1}^{|\sset|-1}\J[\BTB_n]\!\!\left(\bmu_{k_i},\bmu_{k_{i+1}}\right)
		+\sum_{i=1}^{|\sset|-1}\J[\BTB_n]\!\!\left(\bmu_{k_i},\bmu_{k_{i+1}}\right)\log{\frac{|\BTB_n|}{\alpha}}\,.
		\end{align}
		Recall proposition~\ref{pro:hbasis} (that $|\cC_{\bu}| \le 2 \log_2{(\frac{n}{2})}(\cutS+1) $) and since $\pi_t = \frac{1}{|\cC_{\bu}|}$, then for any labeling $\bu_t\in\{-1,1\}^n$ of cut-size $\cutS[\bu_t]$ we have $\frac{1}{2 (\cutS[\bu_t] +1) \log \frac{n}{2}}\leq \pi_t \leq \frac{1}{\cutS[\bu_t] + 1} $. We then apply the same argument upper bounding $\frac{1}{\alpha}$, 
		\begin{align*}
		\frac{1}{\alpha} &= 1 + \frac{\sum\limits_{t=1}^{T}\frac{1}{\pi_t}}
		{\sum\limits_{i=1}^{|\sset|-1}\J[\BTB_n]\!\!\left(\bmu_{k_i},\bmu_{k_{i+1}}\right)}\\
		&\leq 1 + \frac{2n\log{\left(\frac{n}{2}\right)}T}{|\sset|-1}
		\leq \frac{2n\log{\left(\frac{n}{2}\right)}T + |\sset|-1}{|\sset|-1}
		\leq \frac{\left(2n\log{\left(\frac{n}{2}\right)}+1\right)T}{|\sset|-1}\,,
		\end{align*}
		and substituting $\frac{1}{\alpha}\leq\frac{(2n\log{(\frac{n}{2})}+1)T}{|\sset|-1}$ into the last term of~\eqref{TheoremBasicBound-alpha-sub-Bn} gives
		\begin{align*}
		M_{\BTB_n} &\leq
		2\log_2{\frac{n}{2}}\left(\cut_1 + 1\right)\log{|\BTB_n|} + \\
		&\quad\sum_{i=1}^{|\sset|-1}\J[\BTB_n]\!\!\left(\bmu_{k_i},\bmu_{k_{i+1}}\right)\left[c + \log{|\BTB_n|} + \ln{2n} + \log{\frac{T}{|\sset|-1}} + \log{\log{n}}\right]\,.
		\end{align*}
		Applying the same recursive technique as above yields a bound of 
		\begin{equation*}
		M_{\BTB_n} \leq \cO \left(\cut_1\left(\log{n}\right)^2 + \sum_{i=1}^{|\sset|-1}\J[\BTB_n]\!\!\left(\bmu_{k_i},\bmu_{k_{i+1}}\right)\left(\log{n} + \log{|\sset|}+ \log{\log{T}}\right)\right)\,.
		\end{equation*}
		Using the same argument given in proposition~\ref{pro:hbasis} for any two labelings $\bu,\bu'\in\{-1,1\}^n$, for two consistent well-formed comparators $\bmu, \bmu' \in\Delta_{|\BTB_n|}$ respectively, and for two consistent well-formed comparators $\hat{\bmu}, \hat{\bmu}'\in\Delta_{|\FB_n|}$, we have that  $\J[\BTB_n]\!(\bmu,\bmu')\leq2\log{\frac{n}{2}}\J[\FB_n]\!(\hat{\bmu}, \hat{\bmu}')$. Finally we use $\J[\FB_n] \leq 2H(\bu,\bu')$ from Proposition~\ref{prop:hamming-bound} to complete the proof.
	\end{proof}

	\section{Proof of Proposition~\ref{co:alpha-t-bounds}}\label{sec:alpha-t-proof}
	\begin{proof}
		Using a time-dependent $\alpha$ we can re-write~\eqref{TheoremBasicBound} as
		\begin{equation}
		M_{\spset}\leq \frac{1}{\pi_1}\log{|\spset|} + \sum_{t=1}^{T}\frac{1}{\pi_t}\log{\frac{1}{1-\alpha_t}} + \sum_{i=1}^{|\sset|-1}\J\!\left(\bmu_{k_i},\bmu_{k_{i+1}}\right)\log{\frac{|\spset|}{\alpha_{k_{i+1}}}}\,,
		\end{equation}
		and letting $\alpha_t:=\frac{1}{t+1}$, and letting $c:= \log_2{e}$, gives the following, 
		\begin{align}
		M_{\spset} &\leq \frac{1}{\pi_1}\log{|\spset|} + \sum_{t=1}^{T}\frac{1}{\pi_t}\log{\frac{1}{1-\frac{1}{t+1}}} + \sum_{i=1}^{|\sset|-1}\J\!\left(\bmu_{k_i},\bmu_{k_{i+1}}\right)\log{\left(|\spset|\left(k_{i+1} + 1\right)\right)}\label{eq:log(1+x)-before}\\
		&\leq \frac{1}{\pi_1}\log{|\spset|} + c\sum_{t=1}^{T}\frac{1}{\pi_t}\frac{1}{t} + \sum_{i=1}^{|\sset|-1}\J\!\left(\bmu_{k_i},\bmu_{k_{i+1}}\right)\log{\left(|\spset|T\right)}\label{eq:log(1+x)-after}\\
		&\leq \frac{1}{\pi_1}\log{|\spset|} + c\left(\max\limits_{t\in[T]}\frac{1}{\pi_t}\right)\sum_{t=1}^{T}\frac{1}{t} + \sum_{i=1}^{|\sset|-1}\J\!\left(\bmu_{k_i},\bmu_{k_{i+1}}\right)\log{\left(|\spset|T\right)}\label{eq:max-pi_t}\\
		&\leq \frac{1}{\pi_1}\log{|\spset|} + \left(\max\limits_{t\in[T]}\frac{1}{\pi_t}\right)\log{\left(eT\right)}+ \sum_{i=1}^{|\sset|-1}\J\!\left(\bmu_{k_i},\bmu_{k_{i+1}}\right) \log{\left(|\spset|T\right)}\label{eq:adaptive-alpha-bound}
		\end{align}
		where the step from~\eqref{eq:log(1+x)-before} to~\eqref{eq:log(1+x)-after}	has used $\log_2{(1+x)}\leq cx$ for $x>0$, and the step from~\eqref{eq:max-pi_t} to~\eqref{eq:adaptive-alpha-bound} has used $\sum\limits_{t\in[T]}\frac{1}{t} < \int_{1}^{T}\frac{1}{t}dt + 1 = \ln{(eT)}=\frac{1}{c}\log_{2}{(eT)}$.
		
		We now use the following upper bound on $\max\limits_{t\in[T]}\frac{1}{\pi_t}$,
		\begin{equation*}
		\max\limits_{t\in[T]}\frac{1}{\pi_t}\leq \frac{1}{\pi_1} + \sum\limits_{i=1}^{|\sset|-1}\J(\bmu_{k_i},\bmu_{k_{i+1}})\,,
		\end{equation*}
		and the assumption that $\sum\limits_{i=1}^{|\sset|-1}\J(\bmu_{k_i},\bmu_{k_{i+1}}) \geq \frac{1}{\pi_1}$, to give
		\begin{equation}
		\max\limits_{t\in[T]}\frac{1}{\pi_t}\leq 2\sum\limits_{i=1}^{|\sset|-1}\J(\bmu_{k_i},\bmu_{k_{i+1}})\,.\label{eq:max-pi-bound}
		\end{equation}
		Substituting~\eqref{eq:max-pi-bound} into~\eqref{eq:adaptive-alpha-bound} then gives
		\begin{align*}
		M_{\spset} &\leq \frac{1}{\pi_1}\log{|\spset|} + 2\sum_{i=1}^{|\sset|-1}\J\!\left(\bmu_{k_i},\bmu_{k_{i+1}}\right)\left(\log{\left(eT\right)}+ \frac{1}{2} \log{\left(|\spset|T\right)}\right)\notag\\
		&= \frac{1}{\pi_1}\log{|\spset|} + 2\sum_{i=1}^{|\sset|-1}\J\!\left(\bmu_{k_i},\bmu_{k_{i+1}}\right)\left(\frac{1}{2} \log{\left(|\spset|\right)} + \log{\left(e\right)}+ \frac{3}{2}\log{\left(T\right)}\right)
		\end{align*}
		
		Using a conservative update (see section~\ref{sec:cluster-specialists}), we similarly set $\alpha_t:= \frac{1}{m+1}$, where $m$ is the current number of mistakes of the algorithm. We next use the same `recursive trick' as that in the proof of Corollary~\ref{co:ss}. The proof follows analogously, leaving 
		\begin{equation*}
		M_{\FB_n} \leq \cO \left(\cut_1\log{n} + \sum_{i=1}^{|\sset|-1}\J[\FB_n]\!\!\left(\bmu_{k_i},\bmu_{k_{i+1}}\right)\left(\log{n} + \log{|\sset|}+ \log{\log{T}}\right)\right)
		\end{equation*}
		for the basis set $\FB_n$, and 
		\begin{equation*}
		M_{\BTB_n} \leq \cO \left(\cut_1\left(\log{n}\right)^2 + \sum_{i=1}^{|\sset|-1}\J[\BTB_n]\!\!\left(\bmu_{k_i},\bmu_{k_{i+1}}\right)\left(\log{n} + \log{|\sset|}+ \log{\log{T}}\right)\right)
		\end{equation*}
		for the basis set $\BTB_n$.		
	\end{proof}

	\section{The Switching Graph Perceptron}\label{sec:perceptron-appendix}
	In this section for completeness we provide the kernelized Perceptron algorithm for switching graph prediction. The algorithm is described and a mistake bound given for the switching-graph labeling problem in~\cite[Sec. 6.2]{Herbster-SwitchingGraphs}.
	
	\begin{algorithm2e}[t]
		\SetAlgoVlined
		\DontPrintSemicolon
		\SetKwInOut{Input}{input}
		\SetKwInOut{Init}{initialize}
		\SetKwInOut{Parameter}{parameter }
		\SetKw{Predict}{predict}
		\SetKw{Receive}{receive}
		\Input{ Graph $\graph$\;}
		\Parameter{ \(\gamma >0\)}
		\Init{ \(\wt[1] \gets \bm{0}\)\;} 
		\(\bK \gets \bL^{+}_\graph + \max_{i\in [n]} (\be^{\top}_{i} \bL^{+}_\graph \be_{i})\bm{1}\bm{1}^{\top}\)\;
		\For{$t=1$ \KwTo $T$}{
			\Receive{\(i_{t}\in V\)\;}
			\Predict{\(\predyt \gets \text{sign}(w_{t,{i_t}})\)\;}
			\Receive{\(y_{t}\in\{-1,1\}\)\;}
			\If{\(\predyt\neq y_{t}\)}{
				\(\wtm \gets \wt + \yt\frac{\bK\be_{i_t} }{\bK_{i_t,i_t}}\hfill\)\;
				\eIf{\(\|\wtm\|_{\bm{K}} > \gamma\)}{
					\(\wt[t+1] \gets \frac{\wtm}{\|\wtm\|_{\bm{K}}}\gamma\)\;
				}{
					\(\wt[t+1] \gets \wtm\)
				}
			}
		}	\caption{\mbox{\small{\sc Switching Graph Perceptron}}}\label{SGP-algorithm}
	\end{algorithm2e}

	The key to the approach is to use the following graph kernel  (introduced by~\cite{Herbster-GraphPerceptron}) $\bK := \bm{L}^{+}_{\graph}+R_L\bm{1}\bm{1}^{\top}$ with $R_L := \max_{i} (\be^{\top}_{i} \bL^{+}_\graph \be_{i})$, where $\bm{L}^{+}_{\graph}$ denotes the pseudo-inverse of the graph Laplacian, and for $i\in [n]$, we let $\be_i$ denote the $i$-th unit basis vector, i.e., $e_{i,i'}=0$ if $i\neq i'$ and equals $1$ if $i'=i$. The norm induced by this kernel is denoted $\|\bm{u}\|_{\bm{K}} := \sqrt{\bu^{\transpose} \bK^{-1} \bu}$.
	


	\section{Further Details on Experiments}\label{sec:appendix-experiment-details}
	\begin{table*}[!th]
	\begin{center}
		\caption{Parameter ranges used for optimizing the three algorithms with tunable parameters.}
		\label{table:parameter-ranges}
		\begin{tabular}{lccc}
			\toprule
			Algorithm & Parameter & Parameter Range & Optimized Parameter\\
			\midrule
			Kernel Perceptron & $\gamma$ & $3.5- 5$ &  $3.89$\\
			\FS 				& $\alpha$ & $1\times10^{-12} - 1\times10^{-6}$& $7.4\times10^{-10}$\\
			\BS					& $\alpha$ & $1\times10^{-5} - 5\times10^{-4}$& $3.0\times10^{-4}$\\
			\bottomrule
		\end{tabular}
	\end{center}
	\end{table*}

	In this section we give further details on the experimental methods of Section~\ref{sec:Experiments}. Data was collected spanning 72 hours from $4:55am$ on $8^{th}$ April $2019$ to $4:55$am on $11^{th}$ April $2019$. Any stations that were not in service during any of the $72$ hours were removed (in this case there was only one such station). 
	
	As described in Section~\ref{sec:Experiments}, the variable measured was the percentage of occupied docks in each station, and a threshold of $50\%$ was set to induce a binary labeling. Any stations whose induced labeling did not change over the $72$ hours were also removed from the dataset. This left a graph of $404$ stations.
	
	The first $24$ hours of data were used for parameter selection. 
	Parameters were tuned using exhaustive search over the ranges specified in Table~\ref{table:parameter-ranges}, taking the mean minimizer over $10$ iterations. 

	\begin{table*}[!th]
		\begin{center}
			\begin{tiny}
			\caption{Mean error $\pm$ std over $25$ iterations on a $404$-vertex graph for all algorithms and benchmarks, and for all ensemble sizes of \FS and \BS. }        
			\label{table:results}
			\begin{tabular}{l|ccccccc}
				\toprule
				\multicolumn{1}{c|}{}    &   \multicolumn{7}{c}{Ensemble Size}\\
				\multicolumn{1}{c|}{Algorithm} & \multicolumn{1}{c}{1}    &   \multicolumn{1}{c}{3}   &   \multicolumn{1}{c}{5}   & \multicolumn{1}{c}{9} & \multicolumn{1}{c}{17} & \multicolumn{1}{c}{33}& \multicolumn{1}{c}{65}\\
				\midrule 
				\FS&	 $1947 \pm 49$	&	 $1597 \pm 32$	&	 $1475 \pm 30$	&	 $1364 \pm 28$	&	 $1293 \pm 26$	&	 $1247 \pm 21$	&	$1218 \pm 19$\\
				\BS&	 $1438 \pm 32$	&	 $1198 \pm 27$	&	 $1127 \pm 25$	&	 $1079 \pm 24$	&	 $1050 \pm 23$	&	 $1032 \pm 22$&$1021 \pm 18$\\
				Kernel Perceptron&	 $3326 \pm 43$	&	 -			&	 -		&	 -		&	 -			&	 -		&-\\
				Local&	 $3411 \pm 55$	&	 -			&	 -		&	 -		&	 -		&	 -		&-\\
				Global&	 $4240 \pm 44$	&	 -			&	 -		&	 -		&	 -		&	 -		&-\\
				Temporal (Local)&	 $2733 \pm 42$	&	 -			&	 -		&	 -		&	 -			&	 -		&-\\
				Temporal (Global)&	 $3989 \pm 44$	&	 -			&	 -		&	 -		&	 -			&	 -		&-\\
				\bottomrule
			\end{tabular}
			\end{tiny}
		\end{center}
	\end{table*}

\end{document}